\documentclass{article}

\usepackage[main, final]{neurips_2026}

\usepackage[utf8]{inputenc}
\usepackage[T1]{fontenc}
\usepackage{hyperref}
\usepackage{url}
\usepackage{booktabs}
\usepackage{amsfonts}
\usepackage{nicefrac}
\usepackage{microtype}
\usepackage{xcolor}

\usepackage{amsmath}
\usepackage{algorithm}
\usepackage{algpseudocode}
\usepackage{wrapfig}
\usepackage{subcaption}
\usepackage{graphicx}
\usepackage{multirow}
\usepackage{amsthm}
\usepackage[most]{tcolorbox}
\tcbuselibrary{theorems}

\title{RL-IF: Tracking Non-local Training Data Influence for Online Reinforcement Learning}

\author{
    \textbf{Shixuan Liu}\thanks{Equal contribution.},
    \textbf{Cheng Tang}\footnotemark[1],
    \textbf{Yuzheng Hu}\footnotemark[1],
    \textbf{Fan Wu},
    \textbf{Han Zhao},
    \textbf{Jiaqi W. Ma}\\[0.6em]
    University of Illinois Urbana-Champaign\\[0.3em]
    {\small\texttt{\{shixuanl,chengt6,yh46,fanw6,hanzhao,jiaqima\}@illinois.edu}}
}

\newtcbtheorem[number within=section]{theorem}{Theorem}
{colback=gray!5,
 colframe=gray!70,
 boxrule=0.5pt,
 arc=2pt,
 fonttitle=\bfseries}
{thm}

\newtcbtheorem[number within=section]{proposition}{Proposition}
{colback=gray!3,
 colframe=gray!35,
 boxrule=0.4pt,
 arc=2pt,
 left=6pt,
 right=6pt,
 top=5pt,
 bottom=5pt,
 fonttitle=\bfseries,
 coltitle=black}
{prop}

\begin{document}

\maketitle

\begin{abstract}
Online reinforcement learning (RL) has achieved significant success in decision-making tasks, but modern online RL remains highly sensitive to the quality of training experience. Understanding the contribution of individual training instances is therefore central to data-centric RL, with applications ranging from experience selection to policy interpretation. Existing data attribution methods for online RL primarily focus on local influence within a single training round, overlooking the non-local data influence across multiple training rounds. In this work, we formalize data attribution in online RL through a trajectory-level leave-one-out quantity, \emph{RL-LOO}, which measures how one training sample influences a downstream target after subsequent online updates have taken place. We then derive \emph{RL-IF}, a first-order estimator that propagates data influence through both optimization effects and policy-induced sampling effects, and provide theoretical guarantees on its approximation accuracy under smoothness assumptions. Our analysis further disentangles these two effects and shows both theoretically and empirically that the full RL-IF estimator can often be well approximated using the optimization effect alone. Building on this observation, we develop \textbf{RINSE}, a practical data filtering method for online RL. Experiments on standard control benchmarks and an RLHF-style toxicity-mitigation setting show that RINSE improves training efficiency and final performance over standard PPO and local attribution baselines. These results suggest that RL-IF  provides a practical and principled framework for understanding and improving online RL.
\end{abstract}

\section{Introduction}

Online reinforcement learning (RL) has become a central paradigm for learning decision-making policies through repeated interaction with an environment, from robotic control~\citep{todorov2012mujoco,andrychowicz2020learning} to the alignment and enhancement of large language models (LLMs)~\citep{ouyang2022training,shao2024deepseekmath}. Despite this success, modern online RL remains highly sensitive to the quality of training experience: training often requires millions of interactions, while the quality of collected experience can vary substantially across time and trajectories, leading to well-known challenges such as sample inefficiency, high variance, and training instability~\citep{yu2018towards,henderson2018deep,dulac2019challenges}. Hence, understanding the contribution of individual training instances from these experiences is fundamental to data-centric studies of RL and has broad implications for applications such as data selection~\citep{hu2026drposttrainingdataregularization,zhu2025data}, experience replay~\citep{schaul2015prioritized}, or interpretation~\citep{hu2025snapshotinfluencelocaldata}.

Data attribution offers a natural lens for understanding the role of individual training instances. In supervised learning, data attribution aims to identify the training examples that most strongly drive model behavior~\citep{koh2017understanding,deng2025survey}, with applications in model debugging~\citep{koh2017understanding}, mislabeled data detection~\citep{pruthi2020estimating}, data valuation~\citep{ghorbani2019data,jia2019towards}, and data selection~\citep{xia2024less}. Recently, \citet{hu2025snapshotinfluencelocaldata} introduce data attribution to online RL, and develop a local attribution framework for PPO. The proposed framework measures the influence of rollout records on a local policy update within a round, and shows that such local scores can identify harmful training time steps and improve sample efficiency through data filtering.

However, the existing attribution methods are not yet principled and capture only a partial notion of data influence in online RL. First, they mainly rely on heuristic influence estimation, such as gradient dot products, which are intuitive and scalable, but do not fully account for the training dynamics or correspond to a clear counterfactual notion of data value~\citep{koh2017understanding}. Second, the scope of existing attribution methods is restricted to a single training round, leaving the influence of training data across rounds largely unexplored. Characterizing such inter-round influence is more challenging in online RL because the effect of a single training occurrence can propagate across future rounds through subsequent optimization updates and, in principle, through changes in the policy-induced distribution of future rollout data.

This motivates the development of a framework that directly estimates counterfactual influence beyond the local scope of individual RL training rounds. To address this, we formulate data attribution in online RL through a principled leave-one-out objective defined over multiple training rounds, and use it to study the non-local value of rollout data. Our main contributions are summarized as follows:

\begin{itemize}
\item \textbf{A non-local attribution framework for online RL (Section~\ref{sec:prelim-problem-formulation}).} Prior online-RL attribution methods are local and heuristic; we formulate non-local data attribution in online RL through an RL leave-one-out objective, \emph{RL-LOO}, defined over the training trajectory.
\item \textbf{A first-order estimator with optimization and sampling effects (Section~\ref{sec:sloo}).}  We derive \emph{RL-IF}, a first-order estimator of RL-LOO, and show that online-RL data influence decomposes into \textbf{optimization effect} and \textbf{sampling effect}. This yields both a conceptual lens for understanding long-horizon influence and a computationally efficient approximation, \emph{RL-IF-opt}, that is well supported by our theory and experiments.
\item \textbf{A practical data-selection algorithm with downstream gains (Section~\ref{sec:data+selection}).} Building on an efficient approximation to RL-IF, we develop \textbf{RINSE}, a non-local attribution-based filtering method for online RL. Across standard benchmark environments and an RLHF-style toxicity-mitigation setting, RINSE improves training efficiency and final performance relative to standard PPO and local filtering baselines.
\end{itemize}

\section{Preliminaries}
\label{sec:sloo-formultion}

\subsection{Background}
\label{sec:prelim-background}
\paragraph{Online RL algorithms.} We consider an online reinforcement learning (RL) process indexed by rounds $t=1,\ldots,T$. At round $t$, the current policy $\pi_{\theta_t}$ interacts with the environment to collect a rollout buffer $D_t$, after which the policy is updated for $M$ gradient steps. In this work, we focus on a specific RL algorithm, proximal policy optimization (PPO)~\citep{schulman2017proximal}. Given the state action pair $(s, a)$, the standard clipped PPO objective maximizes
\begin{equation}
L^{\mathrm{clip}}(\theta)
:=
\mathbb{E}_{(s,a)\sim D_t}\Big[
\min\big(
 r_t(\theta)\hat{A},
 \mathrm{clip}(r_t(\theta),1-\epsilon,1+\epsilon)\hat{A}
\big)
\Big], \quad
r_t(\theta)
:=
\frac{\pi_\theta(a\mid s)}{\pi_{\theta_{t}}(a\mid s)}.
\label{eq:ppo-clip}
\end{equation}
Here, $r_t(\theta)$ is the policy ratio, $\hat{A}$ is an advantage estimate, $\pi_{\theta_{t}}$ is the behavior policy used to collect the rollout, and $\epsilon$ is the clipping threshold. \citet{schulman2017proximal} also introduced an alternative formulation of the PPO objective, in which the clipping mechanism is replaced by KL regularization. Specifically, let $\beta$ be the coefficient of the KL penalty, and the KL-regularized form of PPO maximizes
\begin{equation*}
L^{\mathrm{KL}}(\theta)
:=
\mathbb{E}_{(s,a)\sim D_t}\big[
r_t(\theta)\hat{A}
-\beta \,
\mathrm{KL}\big(
\pi_{\theta_{t}}(\cdot\mid s)
\,\|\,
\pi_\theta(\cdot\mid s)
\big)
\big].
\end{equation*}
Let $P_t(\cdot \mid \theta_t)$ denote the distribution over rollout buffers induced by $\pi_{\theta_t}$, where a rollout buffer $D_t$ is the collection of trajectories sampled in round $t$. Write $D_t^{(j)} \subseteq D_t$ for the $j$-th minibatch in round $t$, and $\theta_t^{(j)}$ for the parameter at the $j$-th gradient step, with $\theta_t = \theta_t^{(0)} = \theta_{t-1}^{(M)}$. For simplicity, we assume that the policy is updated with stochastic gradient descent, and each update takes the form:
\begin{equation}
\theta_t^{(j)}
=
\theta_t^{(j-1)}
-
\frac{\eta_t}{b}
\sum_{x\in D_t^{(j)}}
\nabla_\theta \ell(\theta_t^{(j-1)},x),
\qquad
 D_t \sim P_t(\cdot \mid \theta_{t}),
\label{eq:sloo-update}
\end{equation}

where $\eta_t$ is the learning rate, $b$ is the minibatch size, and $\ell$ denotes the PPO loss.

\begin{figure}[t]
    \centering
    \includegraphics[width=0.99\linewidth]{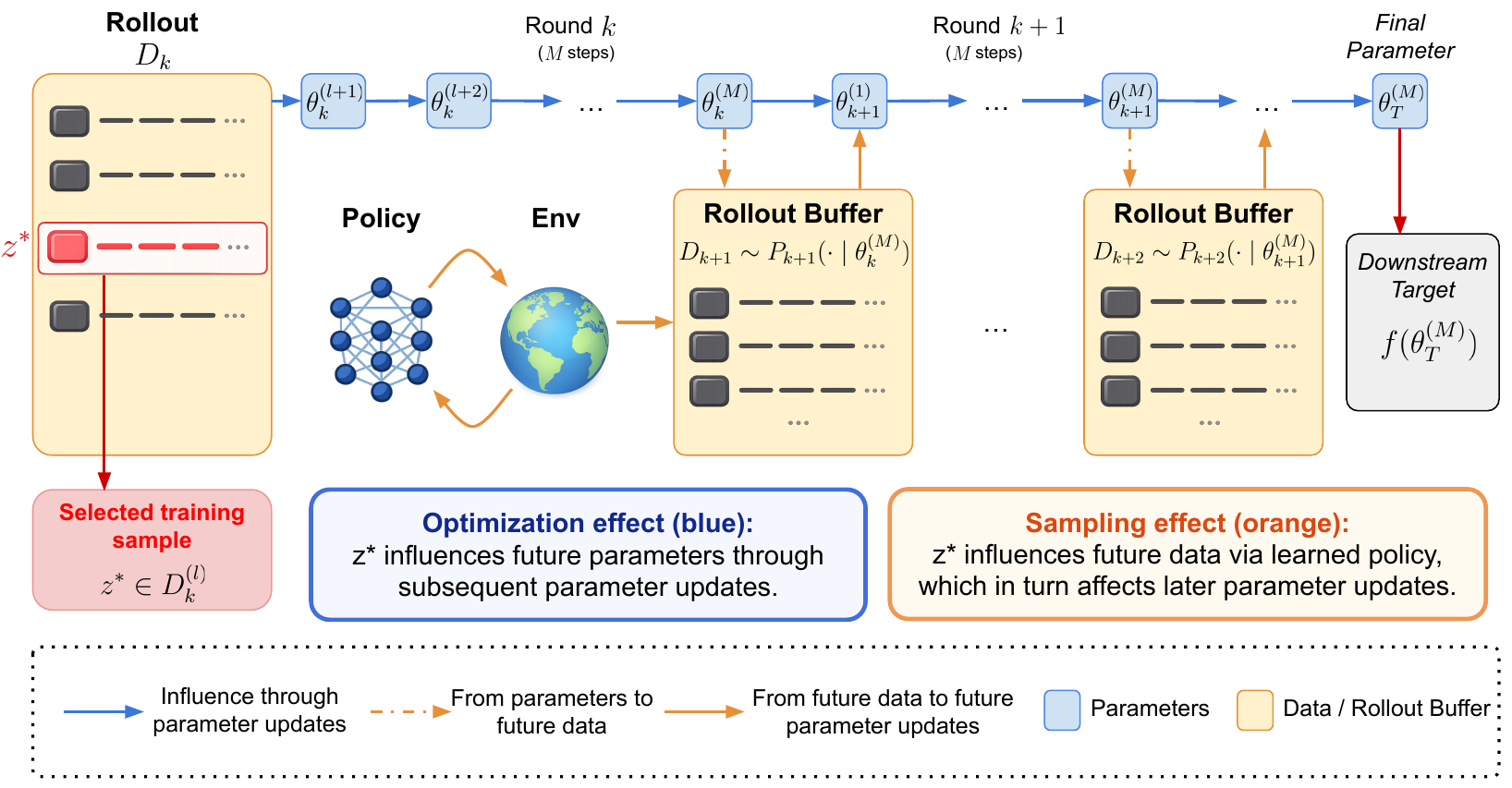}
    \caption{Two-fold data influence in RL. A training occurrence $z^* \in D_k^{(l)}$ affects learning through (i) the \emph{optimization path} (blue): propagating via parameter updates to $\theta_T^{(M)}$, and (ii) the \emph{sampling path} (orange): altering future rollout buffers $D_{k+1}, D_{k+2}, \dots$. These effects are coupled through the policy-induced data distribution.}
    \label{fig:two-fold}
\end{figure}

\paragraph{Local data attribution in online RL.}
\citet{hu2025snapshotinfluencelocaldata} propose a \emph{local} data attribution framework for online RL, which measures how much a training sample contributes to a downstream target function $f$. In this setting, the training samples are scored by gradient inner products between the per-sample gradient and gradient on an attribution target $f$. The local attribution score is estimated as
\begin{equation}
I^{(k,l)}(z^*;f)
:=
-\eta_k \,
\left\langle
\nabla_\theta \ell(\theta, z^*),
\nabla_\theta f(\theta)
\right\rangle,
\label{eq:prelim-tracin}
\end{equation}
A key design choice is the attribution target $f$. \citet{hu2025snapshotinfluencelocaldata} propose to leverage a policy-gradient surrogate for downstream return, defined on a  reference validation distribution. Given a reference policy $\pi^{\mathrm{ref}}$ and a reference advantage estimator $\hat A^{\mathrm{ref}}$, \cite{hu2025snapshotinfluencelocaldata} define
\begin{equation}
f^{\mathrm{return}}(\theta)
:=
\mathbb{E}_{\tau\sim \pi^{\mathrm{ref}},\,(s,a)\sim \tau}
\big[
\log \pi_\theta(a\mid s)\, \hat A^{\mathrm{ref}}(s,a)
\big].
\label{eq:prelim-return-target}
\end{equation}

The target gives the attribution score a direct behavioral interpretation: a positive score means the sample pushes the policy toward higher surrogate return, whereas a negative score suggests that removing the sample may help.

\paragraph{Discrepancy between short-term and long-term influence.}
A major limitation of local data attribution is its myopic view of data influence. To illustrate this point, we empirically estimate the data influence\footnote{We estimate data influence by removing the sample from the training trajectory and measuring the resulting change in the final attribution target, as formally defined in Section~\ref{sec:prelim-problem-formulation}.} in \texttt{LunarLander}~\citep{brockman2016openai}, and track how the influence of the same sample changes as the \emph{horizon} $H$ grows, where horizon means the number of training rounds over which the sample's effect is propagated and evaluated. As shown in Figure~\ref{fig:counter}, a good sample that stabilizes the lander may appear harmful from a short-term viewpoint even though it improves downstream behavior and is beneficial in the long run, and vice versa. Systematically, we also observed a bias induced by the local viewpoint. In particular, local attribution tends to overestimate the long-term value of late-episode samples, those from late time steps in an episode.

\begin{wrapfigure}{r}{0.45\textwidth}
    \centering
    \includegraphics[width=0.45\textwidth]{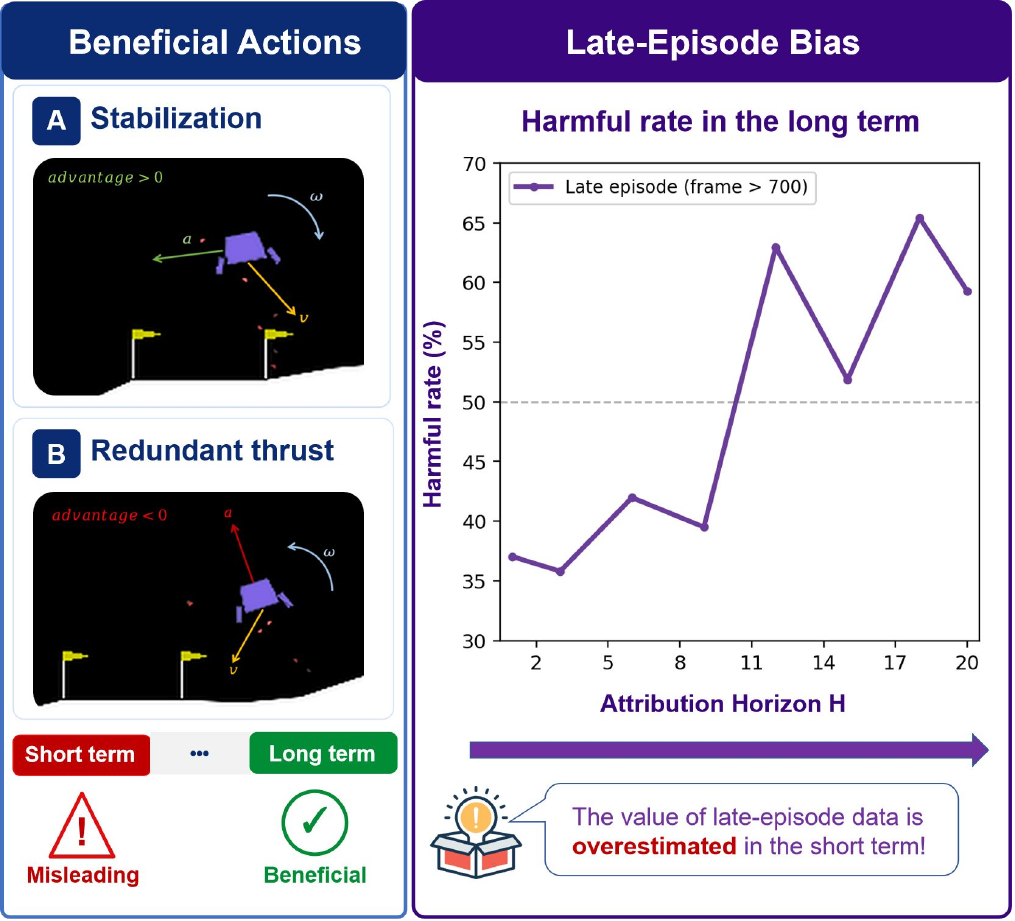}
    \caption{ Left: two LunarLander samples that appear harmful locally but are beneficial in the long run. Right: fraction of late-episode samples labeled harmful vs. attribution horizon $H$. }
\label{fig:counter}
\vspace{-10pt}
\end{wrapfigure}

\paragraph{Two-fold data influence in online RL.}
In online RL, the influence of a training sample is not limited to its direct effect on a parameter update. A rollout affects learning through two coupled channels, as demonstrated in Figure~\ref{fig:two-fold}. First, it influences the checkpoint by contributing to the gradient update that produces the next parameter. This is the usual \textbf{optimization effect} of data: removing one sample perturbs the parameter trajectory and therefore the final checkpoint. Second, the same rollout influences future data sampling. Since subsequent rollout buffers are collected by policies obtained from earlier checkpoints, perturbing one training sample changes not only the next iterate but also the state-action distribution from which later data are drawn. This is referred to as the \textbf{sampling effect}: a sample influences future learning by altering the distribution of later data. Hence, in on-policy RL, data influence is inherently two-fold: it acts on both the optimization path and the sampling path. Moreover, the two effects are generally entangled rather than separable: a perturbation to one sample changes the data distribution of future rollouts, and this distributional shift in turn alters the subsequent optimization trajectory.

\section{Measuring Data Influence in Online RL}
\label{sec:sloo-derivation}

In this section, we systematically formulate and approximate data influence in online reinforcement learning. We first introduce a leave-one-out counterfactual objective, RL-LOO, which measures the long-term effect of removing a training occurrence. Then, we derive a first-order influence estimator, RL-IF, that approximates this counterfactual without retraining by propagating perturbations through the training trajectory. The resulting recursion naturally decomposes data influence into the aforementioned optimization effect and sampling effect, and admits theoretical guarantees under suitable smoothness assumptions. Finally, we analyze both the effectiveness of RL-IF and the relative contributions of these two effects through theoretical analysis and empirical evaluation.

\subsection{Problem Formulation}
\label{sec:prelim-problem-formulation}

Following prior work, we study data attribution in an online PPO-style training pipeline, and extend it beyond individual RL training rounds. Let $\theta_T^{(M)}$ denote the final policy parameter after all rounds of training. Our goal is to systematically measure how one training sample would affect a downstream target after subsequent online updates have taken place.

We formalize this through a leave-one-out (LOO)~\citep{koh2017understanding} counterfactual specific to RL, or \emph{RL-LOO}, following a similar setup in supervised learning literature~\citep{hara2019datacleansingmodelstrained,wang2024capturingtemporaldependencetraining}. The unit of attribution is a particular \emph{occurrence} of a sample in a training minibatch, written as $z^* \in D_k^{(l)}$, which specifies both the sample identity and the minibatch in which it is used. Let $\widetilde{\theta}_T^{(M;-z^*)}$ denote the final parameter obtained by removing this occurrence from the training trajectory and replaying the remaining optimization process, including the induced changes in future sampling. For a downstream target function $f$, RL-LOO is defined as
\begin{equation*}
\mathrm{RL}\text{-}\mathrm{LOO}^{(k,l)}(z^*;f)
:=
f(\widetilde{\theta}_T^{(M;-z^*)})
-
f(\theta_T^{(M)}).
\end{equation*}
This quantity measures the effect of removing one training occurrence on a downstream target after subsequent training. As for the attribution target, we adopt the negative of the target defined in \eqref{eq:prelim-return-target}, so that negative scores identify occurrences whose removal improves the surrogate objective. In the non-local setting, the reference policy is naturally chosen as the final policy $\pi_{\theta_T}^{(M)}$, so that the target measures how removing one training occurrence changes the surrogate return of the final policy on its own induced distribution. $\hat A^{\mathrm{ref}}$ can be approximated using the corresponding critic and the same GAE procedure and hyperparameters as PPO. When exact advantages are used, we show that RL-LOO with this target serves as a local proxy for the change in final expected return in Appendix~\ref{app:reference-target}.

\subsection{Approximating Leave-One-Out in RL}
\label{sec:sloo}

We now derive a first-order estimator, \emph{RL-IF} (RL InFluence), for approximating RL-LOO. At a high level, RL-LOO asks how the final policy would change if a training sample were removed from a rollout-based training trajectory. Rather than explicitly retraining the policy, RL-IF approximates this counterfactual effect through first-order perturbation analysis.

Consider the RL training process in \eqref{eq:sloo-update}. Let $z^* \in D_k^{(l)}$ be a sample removed from update $(k,l)$, and let $\theta_t^{(j)}(\varepsilon)$ denote the perturbed training trajectory obtained by continuously downweighting $z^*$ by a scalar $\varepsilon$. A first-order Taylor expansion gives
\begin{equation}
\mathrm{RL}\text{-}\mathrm{LOO}^{(k,l)}(z^*;f)
\approx
\nabla_\theta f(\theta_T^{(M)})^\top
\Delta_T^{(M)},
\qquad
\Delta_t^{(j)}
:=
\left.
\frac{\partial \theta_t^{(j)}(\varepsilon)}
{\partial \varepsilon}
\right|_{\varepsilon=0},
\label{eq:sloo-delta-definition}
\end{equation}
which motivates the first-order estimator
\begin{equation*}
\mathrm{RL}\text{-}\mathrm{IF}^{(k,l)}(z^*;f)
:=
\nabla_\theta f(\theta_T^{(M)})^\top \Delta_T^{(M)}.
\end{equation*}

The key step is to characterize the perturbation trajectory $\Delta_t^{(j)}$. In online RL, perturbations propagate through two distinct channels, the \emph{optimization effect} and the \emph{sampling effect}, respectively. To unify the effects from both channels, we define the stage update operator
\begin{equation}
F_t^{(j)}(\vartheta,\beta)
:=
\vartheta
-
\frac{\eta_t}{b}
\sum_{x\sim S_t^{(j)}(\beta)}
\nabla_\theta \ell(\vartheta,x),
\quad
S_t^{(j)}(\beta)
:=
\begin{cases}
D_t^{(j)}, & \text{if } t = k \text{ and } j > l, \\
P_t(\cdot \mid \beta), & \text{if } t \ge k+1.
\end{cases}
 \label{eq: stage update operator}
\end{equation}
where $\vartheta$ denotes the parameter being updated and $\beta$ denotes the policy parameter controlling data collection. Differentiating the update recursion with respect to $\varepsilon$ yields
\begin{equation*}
\Delta_t^{(j)}
=
\frac{\partial F_t^{(j)}}{\partial \vartheta}
\Delta_t^{(j-1)}
+
\frac{\partial F_t^{(j)}}{\partial \beta}
\Delta_{t}^{(0)},
\end{equation*}
where the resulting recursion naturally distinguishes the two channels: the first term associated with the current parameters $\vartheta$ captures the optimization effect, and the second term related to the sampling parameters $\beta$ captures the sampling effect. With further simplification, the resulting first-order perturbation recursion is summarized below.

\begin{proposition}{First-order perturbation recursion for RL-IF}{sloo-recursion}
Let $\Delta_t^{(j)}$ denote the perturbation induced by downweighting sample $z^* \in D_k^{(l)}$. The perturbation is initialized as
\[
\Delta_k^{(l)}
=
\frac{\eta_k}{b}
\nabla_\theta \ell(\theta_k^{(l-1)},z^*),
\]
and evolves according to
\begin{equation}
\Delta_t^{(j)}=
\begin{cases}
A_t^{(j)}\Delta_t^{(j-1)},
& \text{if } t=k,\ j>l, \\[4pt]
A_t^{(j)}\Delta_t^{(j-1)}
+
B_t^{(j)}\Delta_t^{(0)},
& \text{if } t\ge k+1,
\end{cases}
\label{eq:rlinf-recursion-main}
\end{equation}
where
\begin{equation*}
A_t^{(j)}
:=
I-\frac{\eta_t}{b}H_t^{(j)},
\qquad
B_t^{(j)}
:=
-\frac{\eta_t}{b}
\sum_{x\in D_t^{(j)}}
\nabla_\theta \ell(\theta_t^{(j-1)},x)
\nabla_\theta \log P_t(x\mid\theta_{t})^\top.
\end{equation*}
Here $A_t^{(j)}$ and $B_t^{(j)}$ correspond to the optimization and sampling effects, respectively.
\end{proposition}

The recursion formula of RL-IF in Proposition~\ref{prop:sloo-recursion} propagates the first-order effect of removing $z^*$ through both later optimization steps and future sampling decisions. The term $A_t^{(j)}\Delta_t^{(j-1)}$ captures the \emph{optimization effect}, while $B_t^{(j)}\Delta_{t}^{(0)}$ captures the additional \emph{sampling effect} caused by policy-dependent data collection. Notably, the sampling effect vanishes when the perturbation remains within the same training round. In practice, we use $\nabla \log \pi_{\theta_{t}}(a\mid s)$ to approximate $\nabla \log P_t((s,a)\mid \theta_{t})$. Note that both matrix multiplications in this formula can be tackled efficiently: multiplying by $A_t^{(j)}$ can be computed using Hessian Vector Product (HVP), and multiplying by $B_t^{(j)}$ is essentially two vector products. We further provide a detailed derivation of the recursion, together with a first-order approximation guarantee for RL-IF under smoothness assumptions, in Appendix~\ref{app:rlloo-derivation} and Appendix~\ref{app:theorem-approx-error}.

\subsection{Sampling vs. Optimization Effects in RL-IF}
\label{sec:sampling-vs-optimization}

A more detailed analysis of RL-IF can be obtained by unrolling the recursion and separating the optimization and sampling effects. In particular, the recursion in \eqref{eq:rlinf-recursion-main} admits a compact closed form.
\begin{equation}
\mathrm{RL}\text{-}\mathrm{IF}^{(k,l)}(z^*;f)
=
\nabla_\theta f(\theta_T^{(M)})^\top
\left(\prod_{t=T}^{k+1} C_t\right)
\left(\prod_{j=M}^{l+1} A_k^{(j)}\right)
\frac{\eta_k}{b}\nabla_\theta \ell(\theta_k^{(l-1)},z^*).
\label{eq:sloo-loss-closed}
\end{equation}
Here, for each stage $t>k$,
\begin{equation*}
A_t := \prod_{j=M}^{1} A_t^{(j)},
\qquad
C_t :=
A_t
+
\sum_{q=1}^{M}
\left(\prod_{j=M}^{q+1} A_t^{(j)}\right) B_t^{(q)}.
\end{equation*}

Equation~\eqref{eq:sloo-loss-closed} is the closed-form solution of the recursion. Here, the product of $A_t$ captures the accumulated \emph{optimization effect}, while the terms involving $B_t^{(j)}$ capture the accumulated \emph{sampling effect}. When $B_t^{(q)}=0$, the expression reduces to the optimization-only case. The following theorem formalizes the gap between the optimization-only and full sampling-aware effects.

\begin{theorem}{Accumulated relative gap}{sloo_acc_gap_nonconvex}
Assume that $\ell(\theta;x)$ is twice differentiable, that its gradient norm is bounded by $G$, that its Hessian norm is bounded by $\Lambda$. Assume in addition that for each future stage $t>k$, the sampling family $P_t(\cdot\mid\theta)$ is KL-smooth with score second moment bounded by $L_{\max}$. Assume the learning rate takes the harmonic form $\eta_t = \frac{\eta_0}{M(t-k)}$ with $\eta_0\Lambda<1$. Then, for any $T>k$ with $H:=T-k$, the accumulated gap satisfies
\begin{equation}
\left\|
\prod_{t=T}^{k+1} C_t
-
\prod_{t=T}^{k+1} A_t
\right\|_2
\le
(e^\alpha H^\alpha-1)
\prod_{t=k+1}^{T}\|A_t\|_2,
\label{eq:sloo-acc-gap-nonconvex}
\end{equation}
where $\alpha := \frac{G\sqrt{L_{\max}}\,\eta_0}{(1-\eta_0\Lambda)^2}$.
\end{theorem}

The bound in \eqref{eq:sloo-acc-gap-nonconvex} shows that the sampling effect remains a \emph{relative perturbation} to the optimization effect. In particular, if $\eta_0\Lambda \ll 1$, the accumulated gap over $H$ stages grows only logarithmically
\begin{equation*}
    e^\alpha H^\alpha-1 \approx \alpha(1 + \log H)
\end{equation*}

\begin{wraptable}{r}{0.55\textwidth}
\vspace{-10pt}
\centering
\footnotesize
\caption{Linear data modeling scores of attribution methods on \texttt{Ant-v4} and \texttt{HalfCheetah-v4}.}
\label{tab:lds-results}
\begin{tabular}{lcc}
\toprule
Score & Ant-v4 LDS $\uparrow$ & HalfCheetah-v4 LDS $\uparrow$ \\
\midrule
Random & $-0.031 \pm 0.058$ & $-0.063 \pm 0.032$ \\
TracIn & $0.055 \pm 0.102$ & $0.138 \pm 0.037$ \\
RL-IF-opt & $0.191 \pm 0.037$ & $0.218 \pm 0.027$ \\
RL-IF-full & $0.192 \pm 0.033$ & $0.218 \pm 0.026$ \\
\bottomrule
\end{tabular}
\end{wraptable}

Hence, in this regime, the accumulated gap remains well-conditioned over long horizons, and the sampling effect acts as a controlled perturbation to the optimization dynamics. Accordingly, we define the optimization-only variant \textbf{RL-IF-opt} as the estimator obtained by setting $B_t^{(j)} = 0$ in the recursion, or equivalently by dropping the sampling terms in the closed-form operator $C_t$. The original estimator that retains both optimization and sampling terms is denoted \textbf{RL-IF-full}. The gap suggests that RL-IF is often dominated by its optimization component, making RL-IF-opt a compelling simplification in practice. The proof of Theorem~\ref{thm:sloo_acc_gap_nonconvex} is deferred to Appendix~\ref{app:acc-bound}.

\paragraph{Experimental comparison.}
To further validate the theoretical findings, we evaluate attribution quality using \emph{Linear Data Modeling Scores} (LDS)~\citep{park2023trak}. LDS measures the Spearman correlation between aggregated attribution scores and the observed effect of removing subsets of training data under retraining. We compare the two variants of RL-IF, along with TracIn~\citep{pruthi2020estimating}, an attribution method that does not account for optimization or sampling dynamics, and a random baseline. We consider KL-regularized PPO on two consecutive training rounds, where retraining is performed with future rollouts resampled using the updated policy. The results are shown in Table~\ref{tab:lds-results}. Both RL-IF-opt and RL-IF-full substantially outperform local attribution methods such as TracIn, while exhibiting nearly identical performance to each other. This small gap suggests that, empirically, the sampling effect contributes only marginally beyond the optimization effect. Additional setup and evaluation details are provided in Appendix~\ref{app:lds-details}.

\section{Data Selection with RL-IF}
\label{sec:data+selection}

We demonstrate a direct downstream application of RL-IF in online RL: using RL-IF attribution scores to guide data selection over multiple training rounds. The goal is to identify and remove training samples that are harmful to downstream performance, thereby improving both compute efficiency and sample efficiency. We build our data selection algorithm on KL-regularized PPO, since the clipping objective in \eqref{eq:ppo-clip} is piecewise-defined and introduces non-smooth behavior at the clipping boundary, which is less compatible with the first-order approximation used in RL-IF. Our source code is available at \url{https://github.com/sx-liu/RL-IF}.

\begin{wrapfigure}{r}{0.58\textwidth}
\vspace{-22pt}
\begin{minipage}{0.57\textwidth}
\begin{algorithm}[H]
\caption{RINSE: RL-IF-based Sample Selection}
\label{alg:onestep-selection}
\begin{algorithmic}[1]
\Require Initial policy $\pi_{\theta_k}$, horizon $H$, discard ratio $\rho$
\State $\theta \gets \theta_k$
\For{$h = 1, \dots, H$}
    \State $D_{k+h-1} \gets \textsc{CollectRollout}(\theta)$
    \State $r(h) \gets \textsc{Evaluate}(D_{k+h-1})$
    \If{$h > 1$ and $r(h-1) > r(h)$}
        \State \textbf{break}
    \EndIf
    \State $\widehat{H} \gets h$
    \State $\theta \gets \textsc{TrainRollout}(\theta; D_{k+h-1})$
\EndFor
\State Compute optimization-only RL-IF-opt scores $\{s(x)\}$ for all samples $x \in \bigcup_{t=k}^{k+\widehat{H}-1} D_t$
\State $\mathcal{R} \gets \textsc{NegBottomFraction}(\{s(x)\}, \rho)$
\State Construct filtered datasets $\widetilde{D}_t = D_t \setminus \mathcal{R}$ for $t = k,\ldots,k+\widehat{H}-1$
\State $\theta \gets \theta_k$
\For{$h = 1, \dots, \widehat{H}$}
    \State $\theta \gets \textsc{TrainRollout}(\theta; \widetilde{D}_{k+h-1})$
\EndFor
\State \Return $\pi_{\theta_{k+\widehat{H}}}$
\end{algorithmic}
\end{algorithm}
\end{minipage}
\vspace{-10pt}
\end{wrapfigure}

\textbf{RINSE: RL-IF-based Sample Selection.} The RL-IF formulation naturally leads to a data-selection algorithm, which we call \textbf{RINSE} (\textbf{R}L-\textbf{IN}fluence-based \textbf{S}ample s\textbf{E}lection), over a finite horizon window. Suppose we aim to improve PPO training for $H$ rounds, starting from round $k$. We first run the standard PPO pipeline for $H$ subsequent rounds, obtaining policies $\pi_{\theta_{k+1}}, \ldots, \pi_{\theta_{k+H}}$ and the corresponding training trajectory $\{D_t^{(j)}\}_{t=k}^{k+H-1}$. This rollout trajectory provides the reference path on which long-horizon influence is evaluated.

Given this trajectory, we estimate the influence of samples in $\{D_t\}_{t=k}^{k+H-1}$ using RL-IF, with the final policy $\pi_{\theta_{k+H}}$ in this window as the reference policy in \eqref{eq:prelim-return-target}. Because the optimization-only approximation achieves comparable empirical performance at substantially lower cost, RINSE uses the optimization-only variant, RL-IF-opt, in the selection procedure. We rank samples by their attribution scores, remove those with the most negative influence based on discard ratio $\rho$, and then retrain on the filtered datasets to obtain updated policies $\pi_{\theta_{k+1}'}, \ldots, \pi_{\theta_{k+H}'}$. Since this retraining stage is inherently off-policy, we use an adjusted retraining objective,
\begin{equation*}
L^{\mathrm{KL\text{-}Retrain}}(\theta)
:=
\mathbb{E}\big[
r_t(\theta)\hat{A}_t
-\beta \,
\mathrm{KL}\big(
\pi_{\theta_{t}'}(\cdot\mid s_t)
\,\|\,
\pi_\theta(\cdot\mid s_t)
\big)
\big],
\end{equation*}
which limits policy drift from the reference policies along the retraining trajectory.

\textit{\underline{Design of attribution horizon.}} The effectiveness of attribution-based selection depends on the attribution horizon and the choice of reference policy. Short horizons lead to overly myopic attribution that reduces to local effects, while excessively long horizons introduce severe off-policy mismatch between the selected data and later training dynamics. Moreover, the reference policy must be sufficiently strong to provide a reliable attribution target.

To balance these considerations, we adopt a \textbf{window-based data selection} strategy, which partitions training into consecutive windows and performs attribution within each window independently. This maintains proximity to the current training dynamics while still capturing non-local effects over multiple future updates.

We further use an \textbf{adaptive horizon selection} procedure. Given a maximum window size $H_{\max}$, we progressively extend the horizon under standard training and monitor the mean episode return $r(h)$ of policy $\pi_{\theta_{k+h-1}}$. We stop when either $h=H_{\max}$ or performance ceases to improve, i.e., $r(h-1) > r(h)$, and select the largest valid horizon for attribution and data selection.

The complete procedure is summarized in Algorithm~\ref{alg:onestep-selection}, while ablations of our design choices are presented in Appendix~\ref{app:ablation}.

\begin{figure*}[t]
    \centering

    \begin{subfigure}[t]{\textwidth}
        \centering
        \caption{\small Test returns over training rounds.\quad \protect\textcolor[HTML]{000000}{\protect\raisebox{0.45ex}{\protect\rule{1.8em}{1pt}}}~Standard\quad \protect\textcolor[HTML]{CC79A7}{\protect\raisebox{0.45ex}{\protect\rule{1.8em}{1pt}}}~Random\quad \protect\textcolor[HTML]{0072B2}{\protect\raisebox{0.45ex}{\protect\rule{1.8em}{1pt}}}~\textsc{IIF}\quad \protect\textcolor[HTML]{D55E00}{\protect\raisebox{0.45ex}{\protect\rule{1.8em}{1pt}}}~RINSE }

        \includegraphics[width=0.32\textwidth]{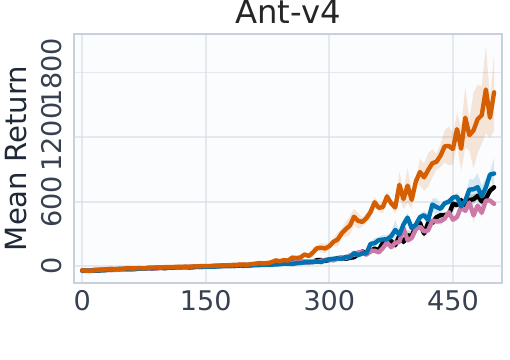}
        \hfill
        \includegraphics[width=0.32\textwidth]{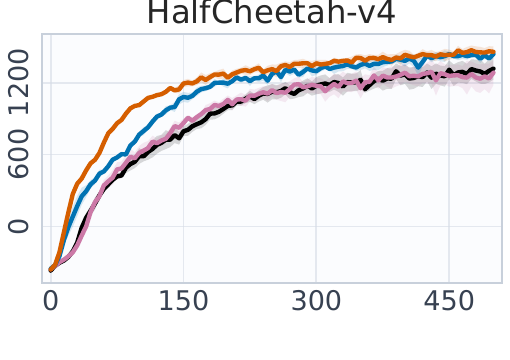}
        \hfill
        \includegraphics[width=0.32\textwidth]{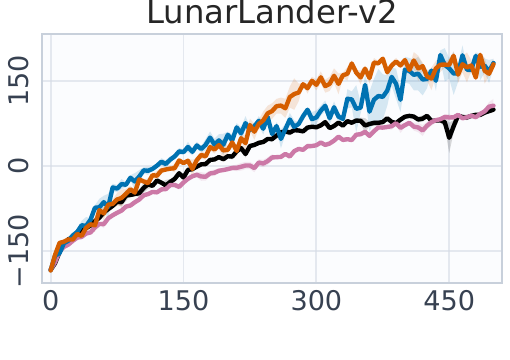}

        \includegraphics[width=0.32\textwidth]{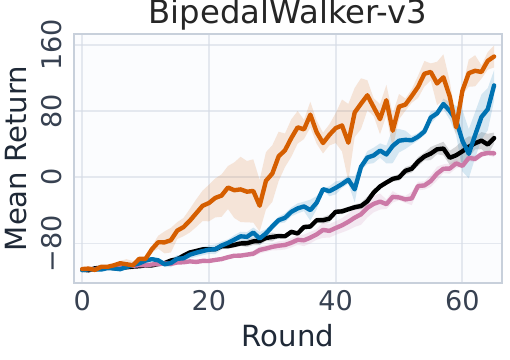}
        \hfill
        \includegraphics[width=0.32\textwidth]{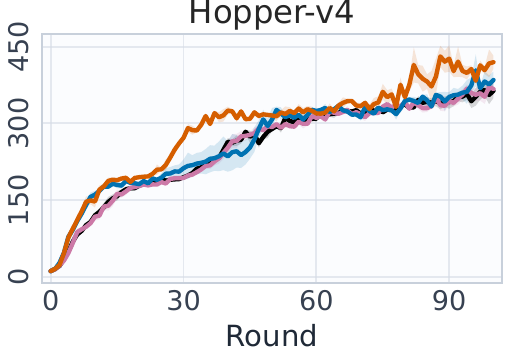}
        \hfill
        \includegraphics[width=0.32\textwidth]{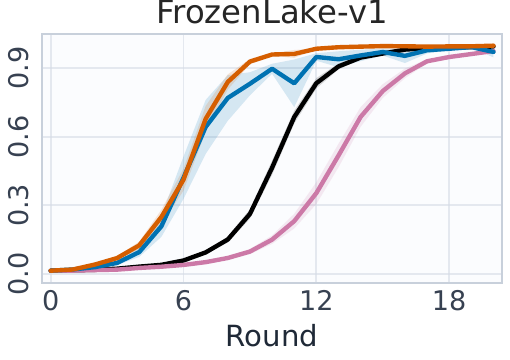}

        \label{fig:data-selection-curves}
    \end{subfigure}

    \vspace{0.8em}

    \begin{subfigure}[t]{\textwidth}
        \centering
        \footnotesize
        \caption{Relative efficiency and peak reward. TE/SE are normalized to standard PPO; lower is better.}
        \setlength{\tabcolsep}{3.5pt}
        \begin{tabular*}{\linewidth}{@{\extracolsep{\fill}}llcccccc@{}}
        \toprule
        Metric & Method & Ant & HalfCheetah & LunarLander & BipedalWalker & Hopper & FrozenLake \\
        \midrule
        \multirow{2}{*}{$\text{TE}\%\downarrow$}
        & IIF & 110.2{\tiny$\pm$3.7} & 75.0{\tiny$\pm$11.0} & \textbf{76.1{\tiny$\pm$8.5}} & 94.9{\tiny$\pm$6.0} & 94.5{\tiny$\pm$18.8} & 133.4{\tiny$\pm$26.0} \\
        & RINSE & \textbf{103.0{\tiny$\pm$4.9}} & \textbf{63.9{\tiny$\pm$10.4}} & 89.9{\tiny$\pm$6.8} & \textbf{67.9{\tiny$\pm$13.6}} & \textbf{88.9{\tiny$\pm$16.7}} & \textbf{116.3{\tiny$\pm$12.5}} \\
        \midrule
        \multirow{2}{*}{$\text{SE}\%\downarrow$}
        & IIF & 86.3{\tiny$\pm$3.3} & 57.7{\tiny$\pm$8.5} & 52.6{\tiny$\pm$6.0} & 74.8{\tiny$\pm$5.0} & 72.2{\tiny$\pm$14.6} & 64.2{\tiny$\pm$12.2} \\
        & RINSE & \textbf{74.9{\tiny$\pm$3.7}} & \textbf{45.4{\tiny$\pm$7.3}} & \textbf{49.8{\tiny$\pm$3.7}} & \textbf{47.1{\tiny$\pm$9.7}} & \textbf{63.9{\tiny$\pm$11.9}} & \textbf{53.5{\tiny$\pm$5.0}} \\
        \midrule
        \multirow{3}{*}{\shortstack{Max\\Reward}$\uparrow$}
        & Standard & 734.3{\tiny$\pm$14.4} & 1316.6{\tiny$\pm$91.6} & 99.1{\tiny$\pm$4.3} & 47.5{\tiny$\pm$6.7} & 365.3{\tiny$\pm$11.6} & 0.99{\tiny$\pm$0.00} \\
        & IIF & 861.6{\tiny$\pm$146.2} & 1437.9{\tiny$\pm$20.2} & 195.0{\tiny$\pm$8.3} & 111.1{\tiny$\pm$18.4} & 403.7{\tiny$\pm$31.3} & 0.99{\tiny$\pm$0.00} \\
        & RINSE & \textbf{1642.8{\tiny$\pm$408.6}} & \textbf{1470.2{\tiny$\pm$39.3}} & \textbf{195.3{\tiny$\pm$3.8}} & \textbf{146.0{\tiny$\pm$13.3}} & \textbf{431.0{\tiny$\pm$20.7}} & \textbf{1.00{\tiny$\pm$0.00}} \\
        \bottomrule
        \end{tabular*}
        \label{fig:data-selection-efficiency}
    \end{subfigure}

    \caption{ Data-selection results under rollout filtering. (a)~Test returns over training rounds. (b)~Relative time efficiency, sample efficiency, and peak mean return across environments. RINSE largely improves return and efficiency over standard training and local baselines. }
    \label{fig:data-selection-results}
    \vspace{-20pt}
\end{figure*}

\paragraph{Experimental setup.}
We evaluate our method on six standard RL benchmarks spanning continuous control and navigation tasks from the Gym/Gymnasium ecosystem, including MuJoCo domains~\citep{todorov2012mujoco,brockman2016openai,towers2024gymnasium}. We use PPO with a KL-regularized objective as the underlying training algorithm and fix the maximum horizon to $H_{\max}=5$ without tuning. We compare \textsc{RINSE} against three baselines: standard PPO, \textbf{Random} filtering, and \textbf{IIF}~\citep{hu2025snapshotinfluencelocaldata}, a local-attribution filtering method based on TracIn scores. For fairness, Random and IIF use the same discard ratio $\rho$ as \textsc{RINSE}.

We report both \textbf{time efficiency (TE)} and \textbf{sample efficiency (SE)} relative to standard PPO. For each method, we measure the runtime and the amount of training data required to reach a set of target performance levels, normalized such that the PPO baseline corresponds to $100\%$. We also report full training curves. Additional setup and metric details are provided in Appendix~\ref{app:experiment-rl} and Appendix~\ref{app:efficiency-details}.

\paragraph{Results.}
Figure~\ref{fig:data-selection-results} shows that \textsc{RINSE} consistently outperforms both standard PPO and IIF. Although \textsc{RINSE} invokes \textsc{TRAINROLLOUT} twice per round and incurs additional influence-computation overhead, it reduces the number of required rollouts by an average of $42.8\%$ to achieve the same performance as standard training. Consequently, \textsc{RINSE} remains more efficient than the local-attribution baseline on nearly all environments.

\textsc{RINSE} achieves the best sample efficiency and highest peak return across all six tasks, with especially large improvements on challenging continuous-control environments such as \texttt{Ant} and \texttt{BipedalWalker}. These results suggest that RL-IF-opt captures the most useful component of long-term data influence for multi-round training improvement. We further report the wall-clock breakdown in Appendix~\ref{app:experiment-rt}.

\subsection{Extension to RLHF}
We also consider an RLHF-style toxicity-mitigation task. Starting from GPT-Neo 2.7B~\citep{black2021gpt}, each training round collects prompt-response trajectories from the current policy, assigns a toxicity-based reward to each response using a pretrained classifier~\citep{vidgen2021lftw,gehman2020realtoxicityprompts}, and then updates the policy following standard RLHF pipelines~\citep{ouyang2022training}. We use the KL-regularized PPO objective described in Section~\ref{sec:prelim-background}. For additional efficiency, we use the Generalized Gauss--Newton matrix (GGN)~\citep{martens2020new,wang2024capturingtemporaldependencetraining} to approximate the Hessian term in \eqref{eq:rlinf-recursion-main}, completely reducing the HVP to efficient vector product operations. Further details are in Appendix~\ref{app:experiment-rlhf}.

\begin{wrapfigure}{r}{0.45\textwidth}
    \vspace{-20pt}
    \centering
    \includegraphics[width=0.45\textwidth]{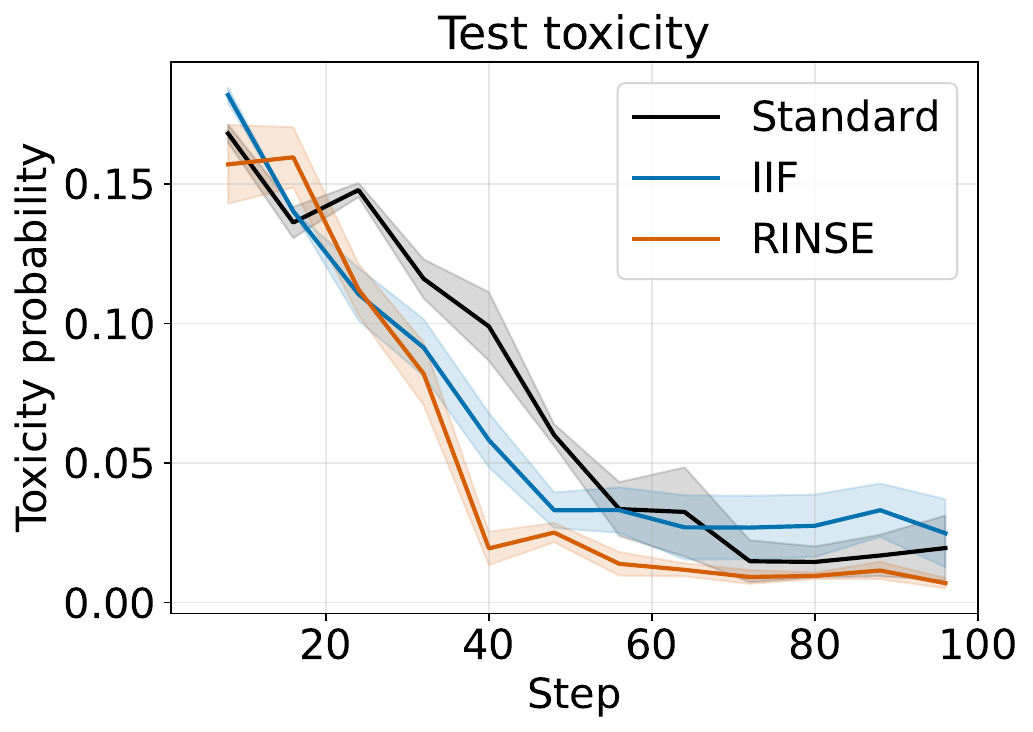}

    \includegraphics[width=0.43\textwidth]{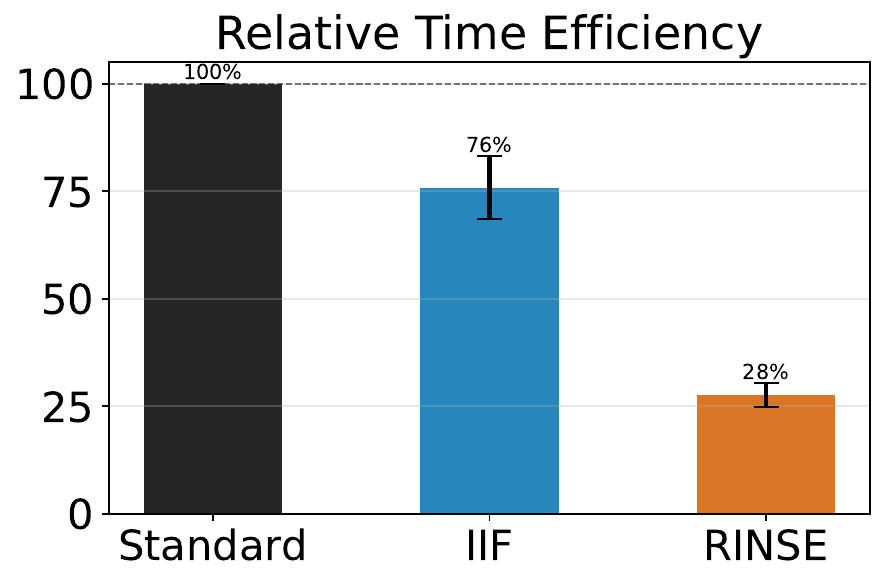}

    \caption{Toxicity-mitigation results in the RLHF setting.}
    \label{fig:rlhf-toxicity-results}
    \vspace{-40pt}
\end{wrapfigure}

\paragraph{Results.} Figure~\ref{fig:rlhf-toxicity-results} shows the toxicity trajectory during training: RINSE achieves lower toxicity probability than the baselines across most of training. Notably, in terms of relative TE on reducing test toxicity, RINSE achieves an average training speedup of more than $3\times$ over the standard training pipeline and more than $2\times$ over IIF. We also provide detailed training details, additional analysis and runtime breakdown in Appendix~\ref{app:results-rlhf}. These results suggest that RINSE extends naturally beyond continuous-control benchmarks to large-scale language-model alignment.

\section{Related Work}
\label{sec:related_Work}

\paragraph{Training-data attribution and data valuation.}
Our work is closely related to the literature on training-data attribution and data valuation~\citep{deng2025survey}. Classical influence functions approximate leave-one-out effects through second-order sensitivity analysis~\citep{koh2017understanding}, while later methods improve scalability by tracing gradient alignment along optimization trajectories, as in TracIn~\citep{pruthi2020estimating}, by using representer-point decompositions~\citep{yeh2018representer}, or by designing scalable estimators for modern large-scale models~\citep{park2023trak}. Recent work further studies attribution in temporally dependent or sequential training settings~\citep{wang2024capturingtemporaldependencetraining}. A closely related line studies \emph{data valuation}, where training examples are assigned value according to their contribution to model quality or downstream utility~\citep{ghorbani2019data,jia2019towards,kwon2021beta}. Other work uses per-example influence signals for data cleaning, robust reweighting, and data selection~\citep{hara2019datacleansingmodelstrained,xia2024less,wang2024greats}.

Most existing attribution methods assume a fixed training distribution. In contrast, our work studies attribution in online RL, where the training distribution evolves with the policy itself. We focus specifically on \emph{non-local} attribution: instead of measuring only the immediate effect of a sample on a nearby update, we study how perturbations propagate through future optimization and future data collection over an extended training trajectory.

\paragraph{Data-centric perspectives in RL.}
A growing line of work studies RL from a data-centric perspective, asking which trajectories, transitions, or training contexts are most useful for improving policy learning. In off-policy RL, this question appears naturally in experience replay and replay prioritization, where informative transitions are replayed more frequently to improve sample efficiency~\citep{schaul2015prioritized}. Related ideas also appear in curriculum learning and curriculum design, where training tasks or experiences are structured adaptively to improve learning dynamics~\citep{florensa2017reverse,jiang2021prioritized}. Recent work also explores trajectory filtering, selective data reuse, and replay-buffer management for improving robustness and efficiency in RL training.

Our work is most closely related to recent attribution methods for online RL. In on-policy training, trajectories are generated by the current policy and immediately reused for optimization~\citep{williams1992simple,sutton1999policy,schulman2015trust,schulman2017proximal}. Consequently, perturbing a single training occurrence can influence not only later parameters, but also the future rollout distribution induced by the updated policy. \citet{hu2025snapshotinfluencelocaldata} take an important first step in this direction by proposing a local attribution framework for PPO that measures the influence of rollout records relative to a nearby checkpoint within a fixed rollout buffer. Our work builds directly on this line, but studies a fundamentally different target: instead of estimating the immediate local effect of a sample, we study its \emph{non-local} influence on future learning outcomes over subsequent training rounds.

\section{Conclusion and Limitations}
\label{sec:conclusions}

We study data attribution in online reinforcement learning, where the training data are generated by the evolving policy rather than drawn from a fixed distribution. To capture this setting, we formulate a trajectory-level leave-one-out quantity, \emph{RL-LOO}, and derive a first-order estimator, \emph{RL-IF}, that propagates the effect of one training occurrence through both later optimization and future sampling. This yields a counterfactual perspective on data value in online RL that goes beyond local attribution. Empirically, the resulting optimization-only approximation, \emph{RL-IF-opt}, is already effective enough to support a practical filtering method, \textbf{RINSE}, which improves efficiency and final performance on standard benchmarks and in an RLHF-style toxicity-mitigation setting.

Our work also has several limitations. First, the current development is specialized to PPO-style on-policy training~\citep{schulman2017proximal} and to a first-order approximation of the leave-one-out trajectory. Extending the framework to other online RL algorithms, including closely related trust-region methods such as TRPO~\citep{schulman2015trust} and other settings with substantially different sampling or update structures, remains an important direction. Second, our derivation is based on SGD-style~\citep{amari1993backpropagation} updates, whereas modern deep RL and RLHF pipelines often rely on adaptive optimizers such as Adam~\citep{kingma2014adam}. Although the method works empirically in our Adam-based experiments, a derivation that explicitly accounts for adaptive preconditioning would make the estimator more principled~\citep{deng2026faithfultrajectorybaseddataattribution}. Finally, our analysis and experiments suggest that the sampling term is often smaller than the optimization term in the regimes we study, but this need not hold universally; understanding when sampling effects become dominant is an open question.

\subsection*{Acknowledgements}
This project was partially supported by an ACCESS allocation CIS260068. JM would like to acknowledge support by NSF CAREER Award No. 2542856. YH and HZ are partially supported by an NSF grant No. 2416897 and an NSF CAREER Award No. 2442290.

\bibliographystyle{plainnat}
\bibliography{references}

\appendix

\section{Additional Details on RL-IF}

\subsection{Derivation of RL Leave-One-Out and RL-IF}
\label{app:rlloo-derivation}

In this section, we derive the recursion underlying RL-IF. To estimate the RL-LOO effect of a sample $z^* \in D_k^{(l)}$, we inject a perturbation at round $k$, step $l$, while keeping all subsequent updates unchanged:
\begin{equation}
\theta_k^{(l)}(\varepsilon)
=
\theta_k^{(l-1)}
-
\frac{\eta_k}{b}
\sum_{x\in D_k^{(l)}\setminus\{z^*\}}
\nabla_\theta \ell(\theta_k^{(l-1)},x)
-
\frac{\eta_k}{b}(1-\varepsilon)
\nabla_\theta \ell(\theta_k^{(l-1)},z^*).
\label{eq:rlloo-perturbed-update}
\end{equation}
The injected perturbation at $(k,l)$ is obtained by differentiating \eqref{eq:rlloo-perturbed-update}:
\begin{equation*}
\Delta_k^{(l)}
:=
\left.
\frac{\partial \theta_k^{(l)}(\varepsilon)}{\partial \varepsilon}
\right|_{\varepsilon=0}
=
\frac{\eta_k}{b}
\nabla_\theta \ell(\theta_k^{(l-1)},z^*).
\end{equation*}
To characterize the terminal perturbation $\Delta_T^{(M)}$, we introduce a stage-update operator that makes explicit the dependence on the current iterate and the rollout distribution. For each update $(t,j)$, define
\begin{equation*}
F_t^{(j)}(\vartheta, \beta)
:=
\vartheta
-
\frac{\eta_t}{b}
\sum_{x \sim S_t^{(j)}(\beta)}
\nabla_\theta \ell(\vartheta, x),
\qquad
S_t^{(j)}(\beta)
:=
\begin{cases}
D_t^{(j)}, & \text{if } t = k \text{ and } j > l, \\
P_t(\cdot \mid \beta), & \text{if } t \ge k+1.
\end{cases}
\end{equation*}
Here $D_t^{(j)}$ denotes a fixed minibatch, while $P_t(\cdot\mid \beta)$ denotes the rollout distribution induced by $\beta$. The perturbed trajectory follows $\theta_t^{(j)}(\varepsilon) = F_t^{(j)}\big( \theta_t^{(j-1)}(\varepsilon), \theta_t^{(0)}(\varepsilon) \big).$ Recall the definition of perturbation $\Delta_t^{(j)}$ in \eqref{eq:sloo-delta-definition}, and by the chain rule, we have
\begin{equation}
\Delta_t^{(j)}
=
\frac{\partial F_t^{(j)}}{\partial \vartheta}
\Delta_t^{(j-1)}
+
\frac{\partial F_t^{(j)}}{\partial \beta}
\Delta_{t}^{(0)}.
\label{eq:rlloo-unified-chain}
\end{equation}
\paragraph{Optimization effect.}
Differentiating $F_t^{(j)}$ with respect to the current iterate $\vartheta$ gives $\frac{\partial F_t^{(j)}}{\partial \vartheta} = A_t^{(j)} = I-\frac{\eta_t}{b}H_t^{(j)},$ where $H_t^{(j)} := \sum_{x\in D_t^{(j)}} \nabla_\theta^2 \ell(\theta_t^{(j-1)},x)$.
\paragraph{Sampling effect.}
For $t=k$ and $j>l$, the minibatch is fixed, so $S_t^{(j)}(\beta)=D_t^{(j)}$ and $\frac{\partial F_t^{(j)}}{\partial \beta}=0.$ Thus, by~\eqref{eq:rlloo-unified-chain} we have $\Delta_k^{(j)} = A_k^{(j)}\Delta_k^{(j-1)}$. For future stages $t\ge k+1$, the sampling distribution depends on $\beta$. Using the score-function identity,
\begin{equation*}
\nabla_\beta
\mathbb{E}_{x\sim P_t(\cdot\mid \beta)}[g(x)]
=
\mathbb{E}_{x\sim P_t(\cdot\mid \beta)}
\left[
g(x)\nabla_\beta \log P_t(x\mid \beta)^\top
\right],
\end{equation*}
with $g(x)=\nabla_\theta \ell(\vartheta,x)$, we obtain
\begin{equation*}
\frac{\partial F_t^{(j)}}{\partial \beta}
=
-\eta_t
\mathbb{E}_{x\sim P_t(\cdot\mid \beta)}
\left[
\nabla_\theta \ell(\vartheta,x)
\nabla_\beta \log P_t(x\mid \beta)^\top
\right].
\end{equation*}
Approximating the expectation with the sampled minibatch yields $\frac{\partial F_t^{(j)}}{\partial \beta} \approx B_t^{(j)}$, where
\begin{equation*}
B_t^{(j)}
:=
-\frac{\eta_t}{b}
\sum_{x\in D_t^{(j)}}
\nabla_\theta \ell(\theta_t^{(j-1)},x)
\nabla_\theta \log P_t(x\mid \theta_t^{(0)})^\top .
\end{equation*}
Substituting into \eqref{eq:rlloo-unified-chain}, we obtain the unified recursion $\Delta_t^{(j)} = A_t^{(j)}\Delta_t^{(j-1)} + B_t^{(j)}\Delta_{t}^{(0)},$ where the second term is active only for $t\ge k+1$.

\subsection{Efficient Computation through Adjoint Formulation}

Even with the recursive formulation of RL-IF, a key challenge is efficient computation. Directly materializing $\Delta_t^{(j)}$ for each candidate sample is computationally prohibitive. Instead, we evaluate RL-IF by propagating an adjoint vector backward through the training trajectory. This converts the forward perturbation recursion into a single backward pass involving Hessian--vector products and score-function vector products. The resulting procedure is summarized in Algorithm~\ref{alg:rlinf}.

The algorithm returns $\nabla_\theta f(\theta_T^{(M)})^\top \Delta_T^{(M)}$, which approximates the RL leave-one-out effect $\mathrm{RL}\text{-}\mathrm{LOO}^{(k,l)}(z^*; f)$. Notably, all influence scores across the trajectory can be computed within a single backward pass.

\begin{algorithm}[H]
\caption{One-pass RL-IF}
\label{alg:rlinf}
\begin{algorithmic}[1]
\Require Training trajectory $\{\theta_t^{(j)}, D_t^{(j)}\}$, attribution target $f$, target sample $z^*\in D_k^{(l)}$
\State $U \gets \nabla_\theta f(\theta_T^{(M)})$
\For{$t = T$ down to $1$}
    \State $U_{\mathrm{prev}} \gets 0$
    \For{$j = M$ down to $1$}
        \If{$t = k$ and $j = l$}
            \State \Return $\left\langle U,\; \frac{\eta_k}{b} \nabla_\theta \ell(\theta_k^{(l-1)}, z^*) \right\rangle$
        \EndIf
        \State \textbf{// sampling contribution}
        \State $U_{\mathrm{prev}} \gets U_{\mathrm{prev}} + B_t^{(j)\top} U$
        \State \textbf{// optimization step}
        \State $U \gets \left(I - \frac{\eta_t}{b} H_t^{(j)}\right) U$
    \EndFor
    \State $U \gets U + U_{\mathrm{prev}}$
\EndFor
\end{algorithmic}
\end{algorithm}
We prove that Algorithm~\ref{alg:rlinf} returns the first-order RL-IF estimator. Recall the forward perturbation recursion in \eqref{eq:rlinf-recursion-main}. For fixed-minibatch updates, the sampling term is inactive, which is equivalently captured by setting $B_t^{(j)}=0$. The target scalar is $\nabla_\theta f(\theta_T^{(M)})^\top \Delta_T^{(M)}.$ We initialize
\begin{equation*}
U_T^{(M)}
=
\nabla_\theta f(\theta_T^{(M)}).
\end{equation*}
Suppose that, before processing update $(t,j)$ in the backward pass, the current adjoint vector is $U_t^{(j)}$. Then, using \eqref{eq:rlinf-recursion-main}, we have
\begin{align}
\left(U_t^{(j)}\right)^\top \Delta_t^{(j)}
&=
\left(U_t^{(j)}\right)^\top
\left(
A_t^{(j)}\Delta_t^{(j-1)}
+
B_t^{(j)}\Delta_{t}^{(0)}
\right) \nonumber\\
&=
\left(A_t^{(j)\top}U_t^{(j)}\right)^\top
\Delta_t^{(j-1)}
+
\left(B_t^{(j)\top}U_t^{(j)}\right)^\top
\Delta_{t}^{(0)}.
\label{eq:rlinf-adjoint-one-step}
\end{align}

Therefore, the adjoint contribution to the previous minibatch iterate is $U_t^{(j-1)} = A_t^{(j)\top}U_t^{(j)}$, which gives the optimization update:
\begin{equation*}
U
\gets
A_t^{(j)\top}U
=
\left(I-\frac{\eta_t}{b}H_t^{(j)}\right)U,
\end{equation*}
where we use the symmetry of $H_t^{(j)}$. The second term in \eqref{eq:rlinf-adjoint-one-step} contributes to the previous-stage perturbation $\Delta_{t}^{(0)}$ rather than to $\Delta_t^{(j-1)}$. Hence, during the backward sweep over stage $t$, all such cross-stage contributions are accumulated as
\begin{equation*}
U_{\mathrm{prev}}
\gets
U_{\mathrm{prev}}
+
B_t^{(j)\top}U.
\end{equation*}
After all updates in stage $t$ have been processed, the adjoint propagated through the optimization dynamics and the accumulated sampling contribution both act on the same previous-stage perturbation $\Delta_{t}^{(0)}$. Thus they are combined by
\begin{equation*}
U \gets U + U_{\mathrm{prev}}.
\end{equation*}
By repeated application of \eqref{eq:rlinf-adjoint-one-step}, the algorithm maintains the invariant that the current adjoint $U$ satisfies
\begin{equation*}
\nabla_\theta f(\theta_T^{(M)})^\top \Delta_T^{(M)}
=
U^\top \Delta_t^{(j)} + U_{\mathrm{prev}}^\top \Delta_t^{(0)}
\end{equation*}
whenever the backward pass is positioned at iterate $\theta_t^{(j)}$. In particular, when the backward pass reaches the injection point $(k,l)$, $\Delta_k^{(0)}=0$, so we have $\nabla_\theta f(\theta_T^{(M)})^\top \Delta_T^{(M)} = U^\top \Delta_k^{(l)}$. The injected perturbation is $\Delta_k^{(l)} = \frac{\eta_k}{b} \nabla_\theta \ell(\theta_k^{(l-1)},z^*).$ Substituting this expression gives
\begin{equation*}
\nabla_\theta f(\theta_T^{(M)})^\top \Delta_T^{(M)}
=
\left\langle
U,\;
\frac{\eta_k}{b}
\nabla_\theta \ell(\theta_k^{(l-1)},z^*)
\right\rangle .
\end{equation*}
This is exactly the value returned by Algorithm~\ref{alg:rlinf}. Therefore, Algorithm~\ref{alg:rlinf} correctly computes the first-order RL-IF estimator for the RL-LOO effect.

\subsection{Closed-Form RL-IF: Proof of Equation \eqref{eq:sloo-loss-closed}}
In this section, we show that the recursion in \eqref{eq:rlinf-recursion-main} directly leads to a compact closed form. For each stage $t>k$, define
\begin{equation}
A_t
:=
\prod_{j=M}^{1} A_t^{(j)},
\qquad
C_t
:=
A_t
+
\sum_{q=1}^{M}
\left(\prod_{j=M}^{q+1} A_t^{(j)}\right) B_t^{(q)}.
\label{eq:rlinf-stage-operator}
\end{equation}
Here, products are ordered from right to left along the forward trajectory; in particular,
\[
\prod_{j=M}^{1}A_t^{(j)}
=
A_t^{(M)}A_t^{(M-1)}\cdots A_t^{(1)}.
\]
Then
\begin{equation}
\mathrm{RL}\text{-}\mathrm{IF}^{(k,l)}(z^*;f)
=
\nabla_\theta f(\theta_T^{(M)})^\top
\left(\prod_{t=T}^{k+1} C_t\right)
\left(\prod_{j=M}^{l+1} A_k^{(j)}\right)
\frac{\eta_k}{b}
\nabla_\theta \ell(\theta_k^{(l-1)},z^*).
\label{eq:rlinf-loss-closed}
\end{equation}

\paragraph{Proof.}
We prove the closed form in \eqref{eq:rlinf-loss-closed} by induction.

First, consider a future stage $t>k$. Recall $\Delta_t^{(j)} = A_t^{(j)}\Delta_t^{(j-1)} + B_t^{(j)}\Delta_{t}^{(0)}.$ We claim that, for every $r=1,\dots,M$,
\begin{equation}
\Delta_t^{(r)}
=
\left(\prod_{j=r}^{1}A_t^{(j)}\right)\Delta_t^{(0)}
+
\sum_{q=1}^{r}
\left(\prod_{j=r}^{q+1}A_t^{(j)}\right)
B_t^{(q)}\Delta_{t}^{(0)},
\label{eq:rlinf-induction-claim}
\end{equation}
where an empty product is the identity. For $r=1$, \eqref{eq:rlinf-induction-claim} gives
\[
\Delta_t^{(1)}
=
A_t^{(1)}\Delta_t^{(0)}
+
B_t^{(1)}\Delta_{t}^{(0)},
\]
which is exactly \eqref{eq:rlinf-recursion-main}. Assume \eqref{eq:rlinf-induction-claim} holds for some $r-1$. Then
\begin{align*}
\Delta_t^{(r)}
&=
A_t^{(r)}\Delta_t^{(r-1)}
+
B_t^{(r)}\Delta_{t}^{(0)} \\
&=
A_t^{(r)}
\left[
\left(\prod_{j=r-1}^{1}A_t^{(j)}\right)\Delta_t^{(0)}
+
\sum_{q=1}^{r-1}
\left(\prod_{j=r-1}^{q+1}A_t^{(j)}\right)
B_t^{(q)}\Delta_{t}^{(0)}
\right]
+
B_t^{(r)}\Delta_{t}^{(0)} \\
&=
\left(\prod_{j=r}^{1}A_t^{(j)}\right)\Delta_t^{(0)}
+
\sum_{q=1}^{r-1}
\left(\prod_{j=r}^{q+1}A_t^{(j)}\right)
B_t^{(q)}\Delta_{t}^{(0)}
+
B_t^{(r)}\Delta_{t}^{(0)} \\
&=
\left(\prod_{j=r}^{1}A_t^{(j)}\right)\Delta_t^{(0)}
+
\sum_{q=1}^{r}
\left(\prod_{j=r}^{q+1}A_t^{(j)}\right)
B_t^{(q)}\Delta_{t}^{(0)}.
\end{align*}
Thus the claim holds for $r$, and therefore for all $r=1,\dots,M$. Setting $r=M$ gives
\begin{equation*}
\Delta_t^{(M)}
=
\left(\prod_{j=M}^{1}A_t^{(j)}\right)\Delta_t^{(0)}
+
\sum_{q=1}^{M}
\left(\prod_{j=M}^{q+1}A_t^{(j)}\right)
B_t^{(q)}\Delta_{t}^{(0)}.
\end{equation*}
Factoring out $\Delta_t^{(0)}$, we obtain
\begin{equation}
\Delta_t^{(M)}
=
\left[
\prod_{j=M}^{1}A_t^{(j)}
+
\sum_{q=1}^{M}
\left(\prod_{j=M}^{q+1}A_t^{(j)}\right)
B_t^{(q)}
\right]\Delta_{t}^{(0)}
=
C_t\Delta_{t}^{(0)}.
\label{eq:rlinf-stage-map-proof}
\end{equation}
Next, for the injection stage $k$, the perturbation after $(k,l)$ is $\Delta_k^{(l)} = \frac{\eta_k}{b} \nabla_\theta \ell(\theta_k^{(l-1)},z^*)$. For $j=l+1,\dots,M$, the minibatches are fixed, so $\Delta_k^{(j)} = A_k^{(j)}\Delta_k^{(j-1)}$. By another direct induction over $j$, this gives
\begin{equation}
\Delta_k^{(M)}
=
\left(\prod_{j=M}^{l+1}A_k^{(j)}\right)
\Delta_k^{(l)}.
\label{eq:rlinf-stage-k-proof}
\end{equation}

Finally, applying \eqref{eq:rlinf-stage-map-proof} inductively over stages $t=k+1,\dots,T$ yields
\begin{equation*}
\Delta_T^{(M)}
=
\left(\prod_{t=T}^{k+1}C_t\right)
\Delta_k^{(M)}.
\end{equation*}
Substituting \eqref{eq:rlinf-stage-k-proof} and the expression for $\Delta_k^{(l)}$ gives
\begin{equation*}
\Delta_T^{(M)}
=
\left(\prod_{t=T}^{k+1} C_t\right)
\left(\prod_{j=M}^{l+1} A_k^{(j)}\right)
\frac{\eta_k}{b}
\nabla_\theta \ell(\theta_k^{(l-1)},z^*).
\end{equation*}
Multiplying both sides by $\nabla_\theta f(\theta_T^{(M)})^\top$ proves
\begin{equation*}
\mathrm{RL}\text{-}\mathrm{IF}^{(k,l)}(z^*;f)
=
\nabla_\theta f(\theta_T^{(M)})^\top
\left(\prod_{t=T}^{k+1} C_t\right)
\left(\prod_{j=M}^{l+1} A_k^{(j)}\right)
\frac{\eta_k}{b}
\nabla_\theta \ell(\theta_k^{(l-1)},z^*).
\end{equation*}
This completes the proof.

\subsection{Construction and Justification of the Reference Target}
\label{app:reference-target}

For an attribution window starting at round $k$, we use its terminal checkpoint $\bar\theta:=\theta_{k+H}$ as the reference policy. We collect an on-policy rollout $D_{k+H}$ from $\pi_{\bar\theta}$ and compute the reference advantages using the critic associated with this checkpoint and the same generalized advantage estimation (GAE) procedure and hyperparameters as PPO. In particular, writing $V^{\mathrm{ref}}$ for the reference critic, the estimates take the form
\begin{align}
\delta_t^{\mathrm{ref}}
&=r_t+\gamma V^{\mathrm{ref}}(s_{t+1})-V^{\mathrm{ref}}(s_t),
\\
\hat A_t^{\mathrm{ref}}
&=\sum_{j=0}^{L_t-1}(\gamma\lambda)^j\delta_{t+j}^{\mathrm{ref}},
\end{align}
where $L_t$ is the number of available steps used in the estimate, with the same termination masking and truncation bootstrapping as in PPO. The reference rollout, critic, and advantage estimates are held fixed during attribution. This construction does not involve tuning over candidate reference advantages.

The terminal checkpoint is used because RL-LOO evaluates the effect of removing a training occurrence on the policy obtained at the end of the window. The following result explains the connection between the reference target and the expected return of that policy.

\paragraph{Ideal reference target.}
For the formal statement, let $0\leq\gamma<1$ and define the normalized discounted return and discounted state occupancy by
\begin{align}
J(\theta)
&:=(1-\gamma)\mathbb E_{\tau\sim\pi_\theta}
\left[\sum_{t=0}^{\infty}\gamma^t r_t\right],
\\
d^{\pi_\theta}(s)
&:=(1-\gamma)\sum_{t=0}^{\infty}\gamma^t
\Pr_{\pi_\theta}(s_t=s).
\end{align}
The initial-state distribution is fixed. Let $A^{\pi_{\bar\theta}}$ denote the exact advantage for the usual discounted return. Following the negative-target convention in the main text, define
\begin{equation}
f_{\mathrm{ref}}(\phi)
:=-\mathbb E_{\substack{s\sim d^{\pi_{\bar\theta}}\\
a\sim\pi_{\bar\theta}(\cdot\mid s)}}
\left[\log\pi_\phi(a\mid s)A^{\pi_{\bar\theta}}(s,a)\right].
\label{eq:ideal-reference-target}
\end{equation}
The reference distribution and advantages in this expression are fixed as $\phi$ varies. This normalization makes the return-gradient identity explicit; the corresponding statement for unnormalized discounted return differs by the constant factor $1-\gamma$.

\begin{proposition}{Reference target and terminal return}{reference-return}
Assume the policy-gradient identity holds and that $f_{\mathrm{ref}}$ and $J$ are twice continuously differentiable in a neighborhood of $\bar\theta$. Let $\bar\theta_{-z}$ be the terminal parameter after deleting occurrence $z$, and set $\Delta_z:=\bar\theta_{-z}-\bar\theta$. For sufficiently small $\Delta_z$,
\begin{equation}
\mathrm{RL\text{-}LOO}(z;f_{\mathrm{ref}})
=J(\bar\theta)-J(\bar\theta_{-z})
+O(\|\Delta_z\|_2^2).
\label{eq:reference-return-loo}
\end{equation}
\end{proposition}

\paragraph{Proof.}
Since the reference distribution and advantages are fixed, the policy-gradient theorem gives
\begin{align}
\left.\nabla_\phi f_{\mathrm{ref}}(\phi)\right|_{\phi=\bar\theta}
&=-\mathbb E_{\substack{s\sim d^{\pi_{\bar\theta}}\\
a\sim\pi_{\bar\theta}(\cdot\mid s)}}
\left[\nabla_\theta\log\pi_{\bar\theta}(a\mid s)
A^{\pi_{\bar\theta}}(s,a)\right]
\\
&=-\nabla J(\bar\theta).
\end{align}
Taylor expansion at $\bar\theta$ therefore yields
\begin{align}
f_{\mathrm{ref}}(\bar\theta_{-z})-f_{\mathrm{ref}}(\bar\theta)
&=-\nabla J(\bar\theta)^\top\Delta_z+O(\|\Delta_z\|_2^2),
\\
J(\bar\theta_{-z})-J(\bar\theta)
&=\nabla J(\bar\theta)^\top\Delta_z+O(\|\Delta_z\|_2^2).
\end{align}
Combining the two expansions proves \eqref{eq:reference-return-loo}.

Thus, with this negative target, a negative RL-LOO value locally corresponds to an increase in terminal return after deletion. RL-IF estimates this counterfactual through the first-order parameter perturbation. In practice, we replace the population target and exact advantages with the empirical reference rollout and GAE estimates, introducing sampling and advantage-estimation error in addition to the first-order approximation. Sensitivity to the reference construction is evaluated in Appendix~\ref{app:reference-ablation}.

\subsubsection{Choice of Reference Policy and Rollout}
\label{app:reference-ablation}

We examine sensitivity to the reference construction by comparing the default terminal checkpoint and its own rollout, $(\theta_{k+H},D_{k+H})$, with two alternatives. The first uses the initial checkpoint of the attribution window and its corresponding rollout, $(\theta_k,D_k)$, following the IIF-style reference construction. The second retains the terminal checkpoint but uses the union of rollout data collected within the window. Although the latter uses more reference data, these rollouts were generated by earlier policies rather than solely by the terminal policy.

\begin{figure}[t]
    \centering
    {\small \textcolor[HTML]{009E73}{\raisebox{0.45ex}{\rule{1.6em}{1pt}}}~$(\theta_{k+H},\,D_{k+H})$ (default)\quad \textcolor[HTML]{E69F00}{\raisebox{0.45ex}{\rule{1.6em}{1pt}}}~$(\theta_{k},\,D_{k})$\quad \textcolor[HTML]{0072B2}{\raisebox{0.45ex}{\rule{1.6em}{1pt}}}~$(\theta_{k+H},\,\bigcup D)$}\\[0.4em]
    \includegraphics[width=0.49\linewidth]{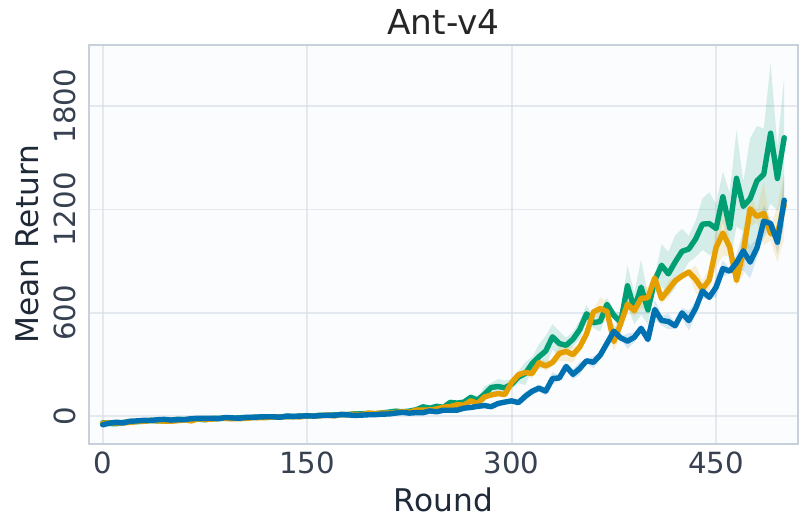}
    \hspace{1mm}
    \includegraphics[width=0.49\linewidth]{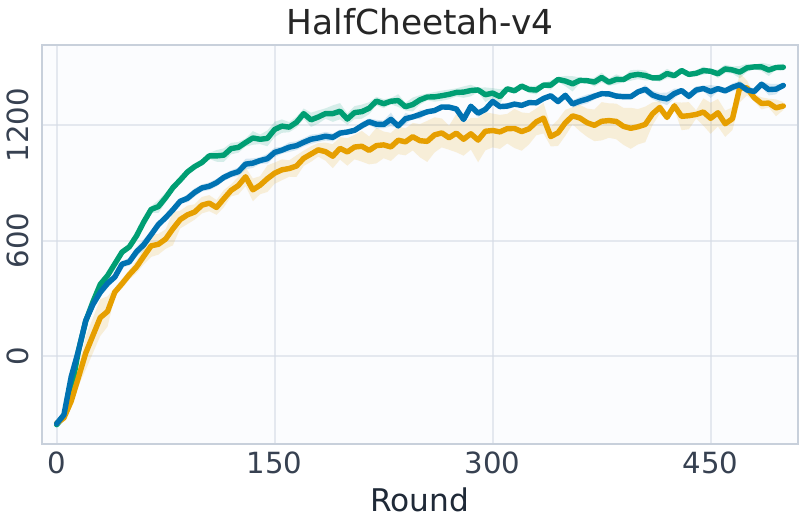}
    \caption{ Reference-construction ablation on Ant and HalfCheetah. Curves show mean return at the reported training rounds; shaded regions indicate one standard error. }
    \label{fig:reference-ablation}
\end{figure}

Figure~\ref{fig:reference-ablation} shows that the default reference achieves the highest final mean return on both environments. These results support pairing the terminal reference policy with its own on-policy rollout.

\section{Detailed Theorem Derivations}
\label{app:theorem-derivations}

In this section, we provide detailed proofs for the theorems.

\subsection{Uniform Approximation Error of First-Order RL-IF}
\label{app:theorem-approx-error}

\begin{theorem}{Uniform approximation error of first-order RL-IF}{RL-IF-uniform-error-bound}
Assume that $\ell(\theta;x)$ is twice differentiable, that its gradient norm is bounded by $G$, that its Hessian norm is bounded by $\Lambda$, and that its Hessian is $L_H$-Lipschitz continuous in $\theta$. Assume also that the rollout distribution $P_t(x\mid\theta)$ has score norm bounded by $S$ and score Hessian norm bounded by $L_P$. Furthermore, assume that the learning rate satisfies $M\overline{\Lambda}\eta_t \leq \frac{1}{T-k}$ for every $k\leq t\leq T$. Then for the first-order RL-IF estimator $\Delta_T^{(M)}$, the approximation satisfies
\[
\bigl\|\widetilde{\theta}_T^{(M)}-\theta_T^{(M)}-\Delta_T^{(M)}\bigr\|_2
\le
\frac{C_{\text{smooth}}}{M^2b^2},
\]
where $C_{\text{smooth}}=\frac{c G^2(L_H+\overline{\Lambda}S+GL_P)}{\overline{\Lambda}^3}$ and $\overline{\Lambda} = \Lambda + GS$ are smoothness constants.
\end{theorem}

\begin{proof}
We analyze the population update map and we do not consider minibatch noise in this theorem. We define the true counterfactual perturbation as
\[
\delta_t^{(j)} := \widetilde{\theta}_t^{(j)}-\theta_t^{(j)}.
\]
The first-order RL-IF perturbation $\Delta_t^{(j)}$ is defined by \eqref{eq:sloo-delta-definition} and satisfies the recursion in Proposition~\ref{prop:sloo-recursion}. At the injection point $(k,l)$, the two coincide exactly:
\[
\delta_k^{(l)}=\Delta_k^{(l)}
=
\frac{\eta_k}{b}\nabla_\theta \ell(\theta_k^{(l-1)},z^*).
\]
Our goal is to bound the discrepancy
\[
e_t^{(j)}:=\delta_t^{(j)}-\Delta_t^{(j)}
\]
at the terminal point $(T,M)$. Recall the stage update map $F_t^{(j)}(\vartheta,\beta)$ defined in \eqref{eq: stage update operator}. By construction, the post-intervention trajectory is generated by
\[
\theta_t^{(j)}(\varepsilon)
=
F_t^{(j)}\!\left(\theta_t^{(j-1)}(\varepsilon),\theta_t^{(0)}(\varepsilon)\right).
\]
Since $\ell(\theta;x)$ has an $L_H$-Lipschitz Hessian, and the rollout distribution $P_t(x\mid\theta)$ has score norm bounded by $S$ and score Hessian norm bounded by $L_P$, the map $F_t^{(j)}(\vartheta,\beta)$ is jointly smooth in $(\vartheta,\beta)$. We next prove a bound for the remainder:
\begin{equation}
\label{eq:Ft-joint-remainder}
\begin{aligned}
\Bigl\|
&F_t^{(j)}(\vartheta+u,\beta+v)
-F_t^{(j)}(\vartheta,\beta)
-\partial_\vartheta F_t^{(j)}(\vartheta,\beta)u
-\partial_\beta F_t^{(j)}(\vartheta,\beta)v
\Bigr\|_2
\le
\frac{\eta_tK}{2}
(\|u\|_2+\|v\|_2)^2,
\end{aligned}
\end{equation}
where $ K:=L_H+\overline{\Lambda}S+GL_P, \overline{\Lambda}:=\Lambda+GS. $ We denote $m_t^{(j)}(\vartheta,\beta)=\mathbb E_{x\sim P_t(\cdot\mid\beta)}[\nabla_\theta \ell(\vartheta,x)]$, for a future round $t>k$, write $m=m_t^{(j)}$ and define
\[
R_m(u,v)
:=
m(\vartheta+u,\beta+v)-m(\vartheta,\beta)
-\partial_\vartheta m(\vartheta,\beta)u
-\partial_\beta m(\vartheta,\beta)v.
\]
We decompose this remainder as
\[
\begin{aligned}
R_m(u,v)
&=
\underbrace{
m(\vartheta+u,\beta+v)-m(\vartheta,\beta+v)
-\partial_\vartheta m(\vartheta,\beta+v)u
}_{(\mathrm{I})}\\
&\quad+
\underbrace{
\bigl(
\partial_\vartheta m(\vartheta,\beta+v)
-\partial_\vartheta m(\vartheta,\beta)
\bigr)u
}_{(\mathrm{II})}\\
&\quad+
\underbrace{
m(\vartheta,\beta+v)-m(\vartheta,\beta)
-\partial_\beta m(\vartheta,\beta)v
}_{(\mathrm{III})}.
\end{aligned}
\]

\paragraph{Bounding (I).}
For fixed $\beta+v$,
\[
\partial_\vartheta m(\vartheta,\beta+v)
=
\mathbb E_{x\sim P_t(\cdot\mid\beta+v)}
\bigl[\nabla_\theta^2\ell(\vartheta,x)\bigr].
\]
Since the loss Hessian is $L_H$-Lipschitz, we have
\[
\|(\mathrm{I})\|_2
\le
\frac{L_H}{2}\|u\|_2^2.
\]

\paragraph{Bounding (II).}
Using the score-function identity, the Hessian bound $\Lambda$, and the score bound $S$, we obtain
\[
\bigl\|
\partial_\vartheta m(\vartheta,\beta+v)
-\partial_\vartheta m(\vartheta,\beta)
\bigr\|_2
\le
\Lambda S\|v\|_2.
\]
Consequently,
\[
\|(\mathrm{II})\|_2
\le
\Lambda S\|u\|_2\|v\|_2.
\]

\paragraph{Bounding (III).}
The sampling derivative is
\[
\partial_\beta m(\vartheta,\beta)
=
\mathbb E_{x\sim P_t(\cdot\mid\beta)}
\left[
\nabla_\theta\ell(\vartheta,x)
\nabla_\beta\log P_t(x\mid\beta)^\top
\right].
\]
Differentiating once more in $\beta$ produces a score-Hessian term and a score outer-product term. The gradient bound $G$, score bound $S$, and score-Hessian bound $L_P$ therefore imply
\[
\|(\mathrm{III})\|_2
\le
\frac{G(L_P+S^2)}{2}\|v\|_2^2.
\]
Combining the three bounds yields
\[
\begin{aligned}
\|R_m(u,v)\|_2
&\le
\frac{L_H}{2}\|u\|_2^2
+\Lambda S\|u\|_2\|v\|_2
+\frac{G(L_P+S^2)}{2}\|v\|_2^2\\
&\le
\frac{K}{2}(\|u\|_2+\|v\|_2)^2,
\end{aligned}
\]
where $K=L_H+\Lambda S+G(L_P+S^2).$ Since $F_t^{(j)}(\vartheta,\beta) =\vartheta-\eta_t m_t^{(j)}(\vartheta,\beta)$, this proves \eqref{eq:Ft-joint-remainder} for $t>k$. For the remaining updates of round $k$, the map is independent of $\beta$, and the fixed-minibatch average has an $L_H$-Lipschitz Jacobian. Its remainder is bounded by $\eta_kL_H\|u\|_2^2/2$, so \eqref{eq:Ft-joint-remainder} holds there as well. Applying \eqref{eq:Ft-joint-remainder} with $u=\delta_t^{(j-1)}$ and $v=\delta_t^{(0)}$ gives
\[
\delta_t^{(j)}
=
A_t^{(j)}\delta_t^{(j-1)}
+
B_t^{(j)}\delta_t^{(0)}
+
r_t^{(j)},
\]
where
\begin{equation}
\label{eq:rt-bound}
\|r_t^{(j)}\|_2
\le
\frac{\eta_tK}{2}
\left(
\|\delta_t^{(j-1)}\|_2+\|\delta_t^{(0)}\|_2
\right)^2.
\end{equation}
Subtracting the first-order recursion yields
\begin{equation}
\label{eq:error-main-recursion}
e_t^{(j)}
=
A_t^{(j)}e_t^{(j-1)}
+
B_t^{(j)}e_t^{(0)}
+
r_t^{(j)}.
\end{equation}
This also applies to the remaining updates of round $k$, with $B_k^{(j)}=0$. We next control the perturbation size. The uniform Jacobian bounds are
\[
\|\partial_\vartheta F_t^{(j)}(\vartheta,\beta)\|_2
\le 1+\eta_t\Lambda,
\qquad
\|\partial_\beta F_t^{(j)}(\vartheta,\beta)\|_2
\le \eta_tGS.
\]
Integrating these bounds along the line segment between the two inputs gives
\[
\begin{aligned}
&\|F_t^{(j)}(\vartheta+u,\beta+v)
-F_t^{(j)}(\vartheta,\beta)\|_2\le
(1+\eta_t\Lambda)\|u\|_2+\eta_tGS\|v\|_2.
\end{aligned}
\]
For round $k$, the second term can be omitted. Thus both the true perturbation and its first-order approximation grow by at most a factor $q_t:=1+\eta_t\overline{\Lambda}$ per update, when measured relative to the maximum perturbation encountered so far. The learning-rate assumption implies
\[
\sum_{t=k}^T M\eta_t\overline{\Lambda}
\le
\frac{T-k+1}{T-k}
\le 2.
\]
Starting from the injection bound $\|\delta_k^{(l)}\|_2=\|\Delta_k^{(l)}\|_2 \le\eta_kG/b$, we therefore obtain
\begin{equation}
\label{eq:uniform-delta-bound}
\begin{aligned}
D
&:=
\max_{(t,j)\text{ at or after }(k,l)}
\left\{
\|\delta_t^{(j)}\|_2,\|\Delta_t^{(j)}\|_2
\right\}
\le
\frac{\eta_kG}{b}
\prod_{t=k}^T q_t^M
\le
\frac{e^2\eta_kG}{b}
\le
\frac{e^2G}{M\overline{\Lambda}b}.
\end{aligned}
\end{equation}
Here we used $T-k\ge1$ in the last inequality. The sampling perturbation at the start of round $k$ is zero. Consequently, \eqref{eq:rt-bound} gives, for every update after the injection,
\[
\|r_t^{(j)}\|_2
\le
R_t,
\qquad
R_t:=2\eta_tKD^2.
\]

Define
\[
E_k:=\max_{l\le j\le M}\|e_k^{(j)}\|_2,
\qquad
E_t:=\max_{0\le j\le M}\|e_t^{(j)}\|_2
\quad(t>k),
\]
and let $a_t:=M\eta_t\overline{\Lambda}$. Taking norms in \eqref{eq:error-main-recursion} and applying the triangle inequality yields
\[
\|e_t^{(j)}\|_2
\le
q_t
\max\left\{
\|e_t^{(j-1)}\|_2,\|e_t^{(0)}\|_2
\right\}
+
R_t.
\]
For $t>k$, using $e_t^{(0)}=e_{t-1}^{(M)}$ and iterating within the round gives
\[
\begin{aligned}
E_t
&\le
q_t^M E_{t-1}
+
R_t\sum_{i=0}^{M-1}q_t^i\\
&\le
e^{a_t}\bigl(E_{t-1}+MR_t\bigr).
\end{aligned}
\]
In round $k$, we instead start from $e_k^{(l)}=0$. The remaining $M-l$ updates satisfy
\[
E_k
\le
R_k\sum_{i=0}^{M-l-1}q_k^i
\le
Me^{a_k}R_k,
\]
where the sum is empty when $l=M$. In particular, the error accumulated during round $k$ is included.

Unrolling the round-wise bounds therefore gives
\[
\begin{aligned}
E_T
&\le
M\sum_{t=k}^T
R_t\exp\left(\sum_{s=t}^T a_s\right)\\
&\le
2e^2MKD^2\sum_{t=k}^T\eta_t.
\end{aligned}
\]
Finally,
\[
\sum_{t=k}^T\eta_t
\le
\frac{T-k+1}{(T-k)M\overline{\Lambda}}
\le
\frac{2}{M\overline{\Lambda}}.
\]
Combining this with \eqref{eq:uniform-delta-bound}, we conclude
\[
\begin{aligned}
\bigl\|
\widetilde{\theta}_T^{(M)}
-\theta_T^{(M)}-\Delta_T^{(M)}
\bigr\|_2
&=
\|e_T^{(M)}\|_2\\
&\le
E_T\\
&\le
\frac{4e^2KD^2}{\overline{\Lambda}}\\
&\le
\frac{
4e^6G^2
\bigl(L_H+\overline{\Lambda}S+GL_P\bigr)
}{
M^2b^2\overline{\Lambda}^{\,3}
}.
\end{aligned}
\]
This proves the result with the universal constant $c=4e^6$.
\end{proof}

\subsection{Relative Bound between Stage-Wise Operators}
\begin{theorem}{Nonconvex relative gap between sampling-aware and optimization-only operators}{}
\label{thm:sloo_fb_gap_nonconvex}
Under the smoothness and bounded-gradient assumptions in Theorem~\ref{thm:RL-IF-uniform-error-bound}, assume in addition that, for each future stage $t>k$, the sampling family $P_t(\cdot\mid\theta)$ is KL-smooth with score second moment bounded by $L_t$, and that $\eta_t\Lambda < 1$.

Then, for every future stage $t>k$,
\begin{equation}
\|C_t - A_t\|_2
\le
\rho_t \|A_t\|_2,
\qquad
\rho_t
:=
\frac{
G\sqrt{L_t}\,
\frac{(1+\eta_t\Lambda)^{M}-1}{\Lambda}
}{
(1-\eta_t\Lambda)^{M}
}.
\label{eq:sloo-gap-rel-ncvx}
\end{equation}
\end{theorem}

\begin{proof}
Fix a future stage $t>k$. Recall that
\[
C_t
=
A_t
+
\sum_{q=1}^{M}
\left(
\prod_{j=M}^{q+1} A_t^{(j)}
\right)
B_t^{(q)},
\qquad
A_t
=
\prod_{j=M}^{1} A_t^{(j)} .
\]
Therefore,
\[
C_t-A_t
=
\sum_{q=1}^{M}
\left(
\prod_{j=M}^{q+1} A_t^{(j)}
\right)
B_t^{(q)} .
\]
Taking operator norms and using submultiplicativity,
\[
\|C_t-A_t\|_2
\le
\sum_{q=1}^{M}
\left\|
\prod_{j=M}^{q+1} A_t^{(j)}
\right\|_2
\|B_t^{(q)}\|_2 .
\]

By smoothness,
\[
\|A_t^{(j)}\|_2
=
\left\|
I-\frac{\eta_t}{b}H_t^{(j)}
\right\|_2
\le
1+\eta_t\Lambda .
\]
Thus,
\[
\left\|
\prod_{j=M}^{q+1} A_t^{(j)}
\right\|_2
\le
(1+\eta_t\Lambda)^{M-q}.
\]
By the bounded-gradient assumption and the KL-smoothness assumption,
\[
\|B_t^{(q)}\|_2
\le
\eta_t G\sqrt{L_t}.
\]
Therefore,
\[
\|C_t-A_t\|_2
\le
\eta_t G\sqrt{L_t}
\sum_{q=1}^{M}
(1+\eta_t\Lambda)^{M-q}.
\]
Using the geometric sum,
\[
\sum_{q=1}^{M}
(1+\eta_t\Lambda)^{M-q}
=
\frac{(1+\eta_t\Lambda)^M-1}{\eta_t\Lambda},
\]
we obtain
\[
\|C_t-A_t\|_2
\le
G\sqrt{L_t}
\frac{(1+\eta_t\Lambda)^M-1}{\Lambda}.
\]

It remains to relate this bound to $\|A_t\|_2$. Since $\eta_t\Lambda<1$, for every vector $v$,
\[
\|A_t^{(j)}v\|_2
=
\left\|
v-\frac{\eta_t}{b}H_t^{(j)}v
\right\|_2
\ge
\|v\|_2-\frac{\eta_t}{b}\|H_t^{(j)}v\|_2
\ge
(1-\eta_t\Lambda)\|v\|_2 .
\]
Applying this bound through the product gives
\[
\|A_t v\|_2
\ge
(1-\eta_t\Lambda)^M\|v\|_2 .
\]
Hence,
\[
\|A_t\|_2 \ge (1-\eta_t\Lambda)^M .
\]
Combining this with the absolute bound yields
\[
\|C_t-A_t\|_2
\le
\frac{
G\sqrt{L_t}
\frac{(1+\eta_t\Lambda)^M-1}{\Lambda}
}{
(1-\eta_t\Lambda)^M
}
\|A_t\|_2 .
\]
This proves \eqref{eq:sloo-gap-rel-ncvx}.
\end{proof}

\subsection{Accumulated Relative Gap}
\label{app:acc-bound}

\begin{theorem}{Accumulated nonconvex relative gap}{sloo_acc_gap_nonconvex_refined}
Under the assumptions of Theorem~\ref{thm:sloo_fb_gap_nonconvex}, for any $T>k$,
\begin{equation}
\left\|
\prod_{t=T}^{k+1} C_t
-
\prod_{t=T}^{k+1} A_t
\right\|_2
\le
\left[
\prod_{t=k+1}^{T}(1+\rho_t)-1
\right]
\prod_{t=k+1}^{T}\|A_t\|_2 .
\label{eq:sloo-product-gap-general}
\end{equation}
Furthermore, if the learning rate takes the harmonic form
\[
\eta_t = \frac{\eta_0}{M(t-k)},
\]
with $\eta_0\Lambda < 1$, and if $L_{\max}:=\sup_{k<t\le T} L_t$, then, with $H:=T-k$,
\begin{equation}
\left\|
\prod_{t=T}^{k+1} C_t
-
\prod_{t=T}^{k+1} A_t
\right\|_2
\le
\left(
e^\alpha H^\alpha-1
\right)
\prod_{t=k+1}^{T}\|A_t\|_2,
\label{eq:sloo-product-gap-harmonic}
\end{equation}
where
\[
\alpha
:=
\frac{G\sqrt{L_{\max}}\eta_0}{(1-\eta_0\Lambda)^2}.
\]
\end{theorem}

\begin{proof}
We first prove the general accumulated product bound. Define
\[
P_C(k,T)
=
\prod_{t=T}^{k+1} C_t,
\qquad
P_A(k,T)
=
\prod_{t=T}^{k+1} A_t .
\]
Using the telescoping expansion for products,
\[
P_C(k,T)-P_A(k,T)
=
\sum_{s=k+1}^{T}
\left(
\prod_{t=T}^{s+1} C_t
\right)
(C_s-A_s)
\left(
\prod_{t=s-1}^{k+1} A_t
\right).
\]
Taking norms gives
\[
\|P_C(k,T)-P_A(k,T)\|_2
\le
\sum_{s=k+1}^{T}
\left(
\prod_{t=s+1}^{T}\|C_t\|_2
\right)
\|C_s-A_s\|_2
\left(
\prod_{t=k+1}^{s-1}\|A_t\|_2
\right).
\]
By Theorem~\ref{thm:sloo_fb_gap_nonconvex},
\[
\|C_s-A_s\|_2
\le
\rho_s\|A_s\|_2.
\]
Moreover,
\[
\|C_t\|_2
\le
\|A_t\|_2+\|C_t-A_t\|_2
\le
(1+\rho_t)\|A_t\|_2.
\]
Substituting these inequalities gives
\[
\|P_C(k,T)-P_A(k,T)\|_2
\le
\sum_{s=k+1}^{T}
\left(
\prod_{t=s+1}^{T}(1+\rho_t)\|A_t\|_2
\right)
\rho_s\|A_s\|_2
\left(
\prod_{t=k+1}^{s-1}\|A_t\|_2
\right).
\]
Factoring out the full product of $\|A_t\|_2$ gives
\[
\|P_C(k,T)-P_A(k,T)\|_2
\le
\left[
\sum_{s=k+1}^{T}
\rho_s
\prod_{t=s+1}^{T}(1+\rho_t)
\right]
\prod_{t=k+1}^{T}\|A_t\|_2 .
\]
Finally, using the identity
\[
\sum_{s=k+1}^{T}
\rho_s
\prod_{t=s+1}^{T}(1+\rho_t)
=
\prod_{t=k+1}^{T}(1+\rho_t)-1,
\]
we obtain
\[
\left\|
\prod_{t=T}^{k+1} C_t
-
\prod_{t=T}^{k+1} A_t
\right\|_2
\le
\left[
\prod_{t=k+1}^{T}(1+\rho_t)-1
\right]
\prod_{t=k+1}^{T}\|A_t\|_2 .
\]

We next prove the harmonic-learning-rate refinement. From Theorem~\ref{thm:sloo_fb_gap_nonconvex},
\[
\rho_t
=
G\sqrt{L_t}
\frac{(1+\eta_t\Lambda)^M-1}{\Lambda(1-\eta_t\Lambda)^M}.
\]
Let $x_t=\eta_t\Lambda$. Since $\eta_t=\eta_0/(M(t-k))$, we have
\[
Mx_t
=
\frac{\eta_0\Lambda}{t-k}
\le
\eta_0\Lambda
<
1.
\]
Using
\[
(1+x_t)^M-1
\le
\frac{Mx_t}{1-Mx_t},
\qquad
(1-x_t)^{-M}
\le
\frac{1}{1-Mx_t},
\]
we get
\[
\rho_t
\le
G\sqrt{L_t}
\frac{M\eta_t}{(1-M\eta_t\Lambda)^2}.
\]
Since
\[
1-M\eta_t\Lambda
=
1-\frac{\eta_0\Lambda}{t-k}
\ge
1-\eta_0\Lambda,
\]
we have
\[
\rho_t
\le
\frac{G\sqrt{L_t}M\eta_t}{(1-\eta_0\Lambda)^2}
=
\frac{G\sqrt{L_t}\eta_0}{(1-\eta_0\Lambda)^2}
\frac{1}{t-k}.
\]
Using $L_t\le L_{\max}$ gives
\[
\rho_t
\le
\frac{G\sqrt{L_{\max}}\eta_0}{(1-\eta_0\Lambda)^2}
\frac{1}{t-k}
=
\frac{\alpha}{t-k}.
\]
Therefore,
\[
\sum_{t=k+1}^{T}\rho_t
\le
\alpha
\sum_{h=1}^{H}\frac{1}{h}
\le
\alpha(1+\log H).
\]
Finally,
\[
\prod_{t=k+1}^{T}(1+\rho_t)-1
\le
\exp\left(\sum_{t=k+1}^{T}\rho_t\right)-1
\le
\exp\!\big(\alpha(1+\log H)\big)-1.
\]
Substituting this into \eqref{eq:sloo-product-gap-general} proves \eqref{eq:sloo-product-gap-harmonic}.
\end{proof}

\section{Detailed Experimental Setup}
\label{app:experimental-setups}

In this section, we provide details of the experimental setup, hyperparameter selection, and implementation details. All experiments are conducted using Amazon Web Services (AWS) c7a high-performance CPUs and 4 NVIDIA A40 GPUs, each with 46GB of memory.

\subsection{Experimental Setup for Traditional RL}
\label{app:experiment-rl}

\paragraph{Environment description.}
Table~\ref{tab:env-summary} summarizes the RL environments used in our experiments. The MuJoCo environments used in this work are distributed under the Apache 2.0 License, while the Gymnasium environments and wrappers, including Box2D and FrozenLake environments, are distributed under the MIT License.

\begin{table}[t]
\centering
\caption{Summary of RL environments used in experiments. MuJoCo environments are from~\citet{todorov2012mujoco}, while Box2D and FrozenLake environments are from Gym/Gymnasium~\citep{brockman2016openai,towers2024gymnasium}.}
\label{tab:env-summary}
\small
\begin{tabular}{lccc}
\toprule
Env & Env ID & Type & Action Space \\
\midrule
Ant & \texttt{Ant-v4} & MuJoCo locomotion & Continuous \\
HalfCheetah & \texttt{HalfCheetah-v4} & MuJoCo locomotion & Continuous \\
Hopper & \texttt{Hopper-v4} & MuJoCo locomotion & Continuous \\
LunarLander & \texttt{LunarLander-v2} & Box2D control & Discrete \\
BipedalWalker & \texttt{BipedalWalker-v3} & Box2D locomotion & Continuous \\
FrozenLake & \texttt{FrozenLake-v1} & Tabular navigation & Discrete \\
\bottomrule
\end{tabular}
\end{table}

\paragraph{Training setup.}
We use \texttt{Stable-Baselines3}~\citep{stable-baselines3} (MIT license) as our training framework for all the traditional RL tasks. We implement a KL-regularized PPO~\citep{schulman2017proximal}, using KL estimated via the $k_3$ approximation~\citep{schulman2020klapprox}
\[
\widehat{\mathrm{KL}}(\pi_{\mathrm{ref}}\,\|\,\pi_\theta) = (r-1) - \log r,\quad r = \pi_\theta(a\mid s)/\pi_{\mathrm{ref}}(a\mid s)
\]

We summarize the detailed hyperparameter selection in table~\ref{tab:ppo-hyperparams}. For base hyperparameters selection, we refer to RL Baselines3 Zoo\footnote{\url{https://github.com/DLR-RM/rl-baselines3-zoo}}, with some minor tuning. We perform a grid search over KL coefficients $\beta$ ranging from 0.0 to 5.0, and select the value that yields the best standard training performance. Similarly, we search for the $\rho$ that produces the highest and most stable performance for the baseline method IIF. For network architecture, we use the default \texttt{MlpPolicy} provided by \texttt{Stable-Baselines3}.

\begin{table}[h]
\centering
\caption{PPO/training hyperparameters per environment.}
\label{tab:ppo-hyperparams}
\resizebox{\linewidth}{!}{%
\begin{tabular}{lcccccc}
\toprule
Hyperparameter & Ant & HalfCheetah & LunarLander & BipedalWalker & Hopper & FrozenLake \\
\midrule
$N_{\text{envs}}$      & 8    & 8    & 1    & 1     & 8    & 1    \\
Total rollouts         & 500  & 500  & 500  & 100   & 100  & 50   \\
$n_{\text{steps}}$     & 512  & 512  & 1024 & 8192  & 512  & 2048 \\
$n_{\text{epochs}}$    & 20   & 20   & 4    & 10    & 20   & 10   \\
Learning rate          & 1e-4 & 1e-4 & 5e-3 & 1e-4  & 1e-4 & 5e-3 \\
Optimizer              & Adam & Adam & SGD  & Adam  & SGD  & SGD  \\
$\gamma$               & 0.98 & 0.98 & 0.99 & 0.99  & 0.98 & 0.99 \\
$\lambda$ (GAE)        & 0.92 & 0.92 & 0.95 & 0.95  & 0.92 & 0.95 \\
$\beta_{\text{KL}}$    & 5.0  & 1.0  & 0.5  & 1.0   & 1.0  & 0.0  \\
$\rho$ (Filtering only)   & 10\% & 10\% & 10\% & 5\%  & 10\% & 50\%  \\
\bottomrule
\end{tabular} }
\end{table}

\paragraph{Evaluation setup.} During evaluation, we evaluate the stochastic performance of each policy averaged over multiple evaluation episodes. For most continuous control environments (Ant, HalfCheetah, Hopper), we average over 32 episodes by default; for BipedalWalker, we use 100 episodes to reduce variance due to its longer horizon; and for light-weight environments, LunarLander and FrozenLake, we use 1000 episodes. Each setup is run with three independent seeds; we report mean~$\pm$~standard error.

\subsection{Experimental Setup for RLHF}
\label{app:experiment-rlhf}

\paragraph{Base model.}
We use GPT-Neo 2.7B~\citep{black2021gpt} (\texttt{EleutherAI/gpt-neo-2.7B}, MIT license) as the base language model. We apply LoRA~\citep{hu2022lora} for parameter-efficient fine-tuning, with rank $r=16$, scaling factor $\alpha=32$, and dropout rate $0.05$.

\paragraph{Training data.}
We train on the Real Toxicity Prompts dataset~\citep{gehman2020realtoxicityprompts} (\texttt{allenai/real-toxicity-prompts}, Apache-2.0 license). We retain examples with a prompt toxicity score above $0.3$, use the 80\% training split, and shuffle with a fixed seed. For each example, we concatenate the prompt and continuation fields, tokenize, and truncate to a length sampled uniformly in $[10, 15]$ tokens to form the PPO prompt. The policy then generates a 30-token continuation with temperature $1.0$ and top-$p = 1.0$.

\paragraph{Reward model.}
The reward signal is provided by the RoBERTa-based hate-speech classifier \texttt{facebook/roberta-hate-speech-dynabench-r4-target}~\citep{vidgen2021lftw}. Following the reference implementation, we use the raw logit of the \texttt{nothate} class (index~$0$) as the scalar reward: a higher logit indicates less toxic output.

\paragraph{Training setup.}
We optimize a PPO objective with a fixed KL penalty and no probability-ratio clipping, with a learning rate of $5\times 10^{-6}$. The per-token KL divergence from the reference policy is penalized with coefficient $\beta = 0.4$, estimated via the $k_1$ approximation~\citep{schulman2020klapprox}. Each rollout collects a batch of $64$ samples; PPO updates use a mini-batch size of $8$ over $4$ epochs, yielding $32$ gradient steps per rollout. We use the Adam optimizer with gradient norm clipping at $1.0$ and no weight decay. For each run, we train for $300$ rollout steps. All runs use bfloat16 precision with Flash Attention~2~\citep{dao2023flashattention}.

For data attribution setups, we use $H=2$ for RINSE, and $\rho=10\%$ for both IIF and RINSE for fairness.

\paragraph{Evaluation setup.}
We evaluate on the Wiki-Toxic dataset~\citep{wulczyn2017ex} (\texttt{OxAISH-AL-LLM/wiki\_toxic}, test split), drawn from a different distribution than the training data. We select the first $200$ toxic examples (label $= 1$) and use up to the first $2{,}000$ characters of each comment as the evaluation prompt---substantially longer than training prompts and thus more likely to elicit toxic continuations. Continuations of $30$ tokens are generated under the same sampling parameters as training. To guard against reward hacking specific to the training signal, we score continuations with a \emph{separate} classifier, \texttt{DaNLP/da-electra-hatespeech-detection}~\citep{pauli2021danlp}, and report the mean toxicity probability across the $200$ samples. Each method is run with three independent seeds; we report mean~$\pm$~standard error.

\section{Additional Experimental Results}
\subsection{Further Details on Linear Data Modeling Scores and Evaluation}
\label{app:lds-details}

Directly validating attribution methods using single-sample leave-one-out (LOO) correlation can be challenging in online RL due to the extremely low signal-to-noise ratio of individual sample effects. In particular, removing a single rollout record often induces only a negligible change in the final policy behavior, while stochasticity from environment interaction introduces substantial variance across retraining runs. To obtain a more stable evaluation protocol, we additionally adopt the linear data modeling score (LDS), which evaluates whether attribution scores can linearly predict the effect of removing a subset of training data. Concretely, LDS samples random subsets of rollout records, sums the attribution scores within each subset, and compares the aggregated score against the actual performance change after retraining without that subset. This evaluation implicitly assumes approximate linear additivity of data influence across samples, and therefore measures how well attribution scores behave as a linear proxy for collective data influence.

\paragraph{Experimental setup.} For the LDS evaluation, we consider model training with KL-regularized PPO on two consecutive rounds. We construct 100 random subsets of training data, each containing 4096 samples. For each subset, we remove it and retrain the model with future rollout data resampled using the updated policy for 32 independent runs, and compute the average change in validation loss. All results are further averaged over 3 random seeds. We conduct the experiments on two MuJoCo benchmark environments~\citep{todorov2012mujoco}, \texttt{Ant-v4} and \texttt{HalfCheetah-v4}.

\begin{table}[H]
\centering
\caption{Linear data modeling scores of attribution methods on \texttt{Ant-v4} and \texttt{HalfCheetah-v4}.}
\label{tab:lds-results-full}
\begin{tabular}{lcccc}
\toprule
& \multicolumn{2}{c}{Ant-v4} & \multicolumn{2}{c}{HalfCheetah-v4} \\
\cmidrule(lr){2-3} \cmidrule(lr){4-5}
Score & Pearson $\uparrow$ & Spearman $\uparrow$ & Pearson $\uparrow$ & Spearman $\uparrow$ \\
\midrule
Random & $-0.020 \pm 0.041$ & $-0.031 \pm 0.058$ & $-0.069 \pm 0.037$ & $-0.063 \pm 0.032$ \\
TracIn & $0.029 \pm 0.081$ & $0.055 \pm 0.102$ & $0.166 \pm 0.043$ & $0.138 \pm 0.037$ \\
RL-IF-opt & $0.179 \pm 0.017$ & $0.191 \pm 0.037$ & $0.266 \pm 0.015$ & $0.218 \pm 0.027$ \\
RL-IF-full & $0.180 \pm 0.023$ & $0.192 \pm 0.033$ & $0.267 \pm 0.016$ & $0.218 \pm 0.026$ \\
\bottomrule
\end{tabular}
\end{table}

Table~\ref{tab:lds-results-full} reports LDS results on \texttt{Ant-v4} and \texttt{HalfCheetah-v4}. We additionally report Pearson correlations for completeness, which exhibit the same overall trend as the Spearman correlations in the main text. Both RL-IF-opt and RL-IF-full substantially outperform TracIn and random baselines across environments and correlation metrics, indicating that the proposed non-local attribution framework better captures the long-term effect of rollout data on downstream policy behavior.

\subsubsection{Direct Leave-One-Out Evaluation}
\label{app:direct-loo}

To complement the subset-removal LDS evaluation, we directly evaluate LOO attribution on FrozenLake with attribution horizon $H=2$. Unlike the fixed-rollout experiment reported separately, this evaluation allows future rollout data to change with the perturbed policy. We use 256 retraining runs to estimate the deletion effects and compare them with the attribution scores of TracIn, RL-IF-opt, and RL-IF-full.

\begin{table}[htbp]
    \centering
    \caption{ Direct LOO attribution evaluation on FrozenLake with $H=2$. Correlations are reported as per-seed means with standard errors. }
    \label{tab:direct-loo-frozenlake}
    \begin{tabular}{lc}
        \toprule
        Method & Pearson $\rho$ (mean $\pm$ SE) \\
        \midrule
        TracIn & $+0.538 \pm 0.165$ \\
        RL-IF-opt & $+0.461 \pm 0.161$ \\
        \textbf{RL-IF-full} & $\mathbf{+0.687 \pm 0.041}$ \\
        \bottomrule
    \end{tabular}
\end{table}

RL-IF-full achieves the highest mean correlation in this experiment, outperforming both RL-IF-opt and TracIn. This provides a setting in which explicitly accounting for policy-induced sampling changes improves agreement with the estimated LOO effects.

\subsubsection{Empirical Characterization of the Sampling Effect}
\label{app:sampling-diagnostics}

To examine why the sampling contribution differs across environments, we measure the top Hessian eigenvalue of the KL-regularized loss, denoted by $\Lambda$, the rollout-score second moment, summarized by $\sqrt{L_{\max}}$, and the learning rate $\eta$ within the attribution window. Table~\ref{tab:sampling-diagnostics} reports these measurements and the diagnostic ratio $\eta\sqrt{L_{\max}}/\Lambda$.

\begin{table}[htbp]
    \centering
    \small
    \caption{Empirical diagnostics of the sampling contribution.}
    \label{tab:sampling-diagnostics}
    \begin{tabular}{lcccc}
        \toprule
        Environment
        & $\Lambda$
        & $\sqrt{L_{\max}}$
        & $\eta$
        & $\eta\sqrt{L_{\max}}/\Lambda$ \\
        \midrule
        FrozenLake
        & $1.3$ & $2.4$ & $5\times10^{-3}$ & $9.2\times10^{-3}$ \\
        HalfCheetah
        & $65$ & $12.0$ & $1\times10^{-4}$ & $1.8\times10^{-5}$ \\
        Ant
        & $2.4\times10^3$ & $12.2$
        & $1\times10^{-4}$ & $5.0\times10^{-7}$ \\
        \bottomrule
    \end{tabular}
\end{table}

FrozenLake has a much smaller measured curvature scale and a larger learning rate than Ant and HalfCheetah. Its reported diagnostic ratio is approximately four orders of magnitude larger than Ant's. Together with the direct LOO comparison in Table~\ref{tab:direct-loo-frozenlake}, these measurements identify FrozenLake as an empirical setting in which retaining the sampling correction is useful. In contrast, the full and optimization-only estimators have nearly identical LDS on Ant and HalfCheetah.

\subsection{Additional Results for LOO}

\begin{figure}[htbp]
    \centering
    \includegraphics[width=\linewidth]{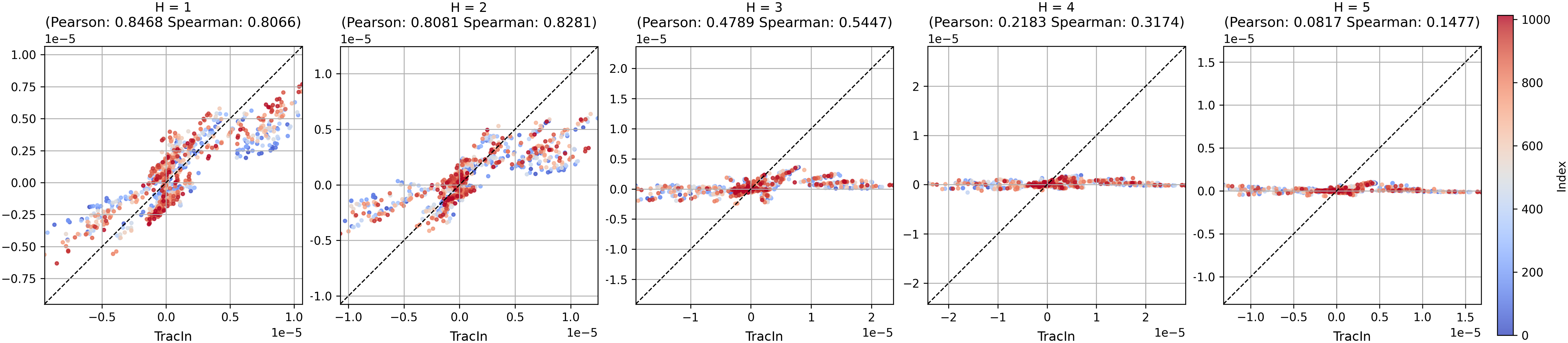}
    \vspace{0.4em}
    \includegraphics[width=\linewidth]{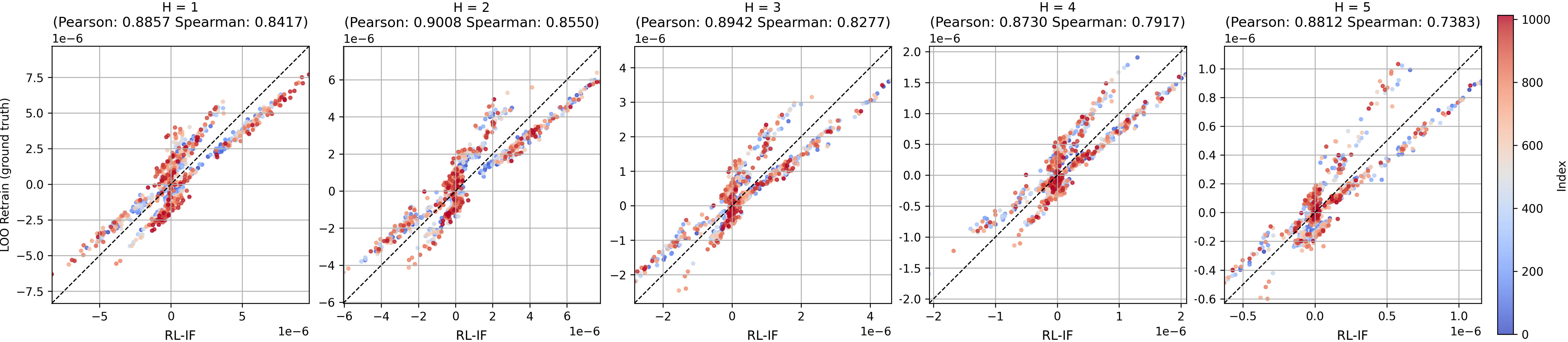}
    \caption{Comparison between TracIn and RL-IF. The top row shows the comparison between the TracIn method and the ground-truth LOO values, while the bottom row shows the comparison between RL-IF and the ground-truth LOO values under the same held-out settings. From left to right, the panels show the attribution horizon $H$ from 1 to 5. The colorbar (point index) is shared across each row and shown once on the right. Each panel keeps its own axes because the attribution and LOO magnitudes differ across horizons. }
    \label{fig:ten_figures}
\end{figure}

Additional LOO results are presented in Figure~\ref{fig:ten_figures}. We compare our RL-IF score with TracIn and ground truth (LOO Retrain), under the fixed rollout setting, where we do not consider the sampling effect. RL-IF shows high correlation with the ground truth LOO retrain, with correlation decaying slowly with respect to the attribution horizon. TracIn shows high correlation with ground truth when attribution horizon equals to 1, but fails to capture the long horizon effect.

\subsection{Ablation on Algorithmic Design}
\label{app:ablation}

In this section, we conduct ablation studies to justify our design choices of the horizon selection $H$ in RINSE.

\subsubsection{Ablation on Adaptive Horizon}

We ablate the effect of the adaptive horizon by adopting a fixed horizon instead, on \texttt{Ant-v4}, \texttt{HalfCheetah-v4}, \texttt{BipedalWalker-v3}, and \texttt{LunarLander-v2}. As shown in Figure~\ref{fig:dynamic}, the adaptive horizon matches or exceeds the fixed variant in every case, and the gap is largest on the harder-exploration tasks: on \texttt{Ant-v4} and \texttt{BipedalWalker-v3} the fixed-horizon runs progress slowly and erratically, while on \texttt{LunarLander-v2} the adaptive horizon reaches a higher return with lower variance. This shows that the adaptive horizon is necessary: it stabilizes training and yields consistently higher-quality data selection without per-environment tuning.

\begin{figure}[t]
  \centering
  {\small \textcolor[HTML]{E69F00}{\raisebox{0.45ex}{\rule{1.6em}{1pt}}}~RINSE + Fixed $H$\qquad \textcolor[HTML]{009E73}{\raisebox{0.45ex}{\rule{1.6em}{1pt}}}~RINSE + Dynamic $H$}\\[0.4em]
  \includegraphics[width=0.49\linewidth]{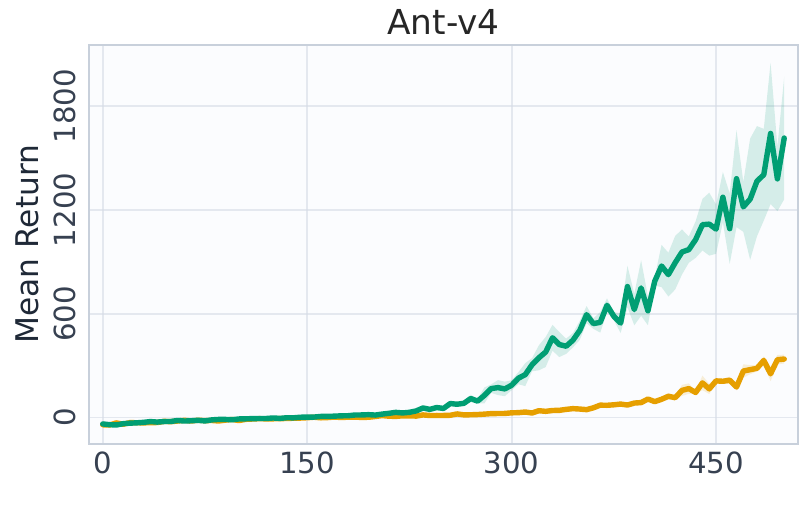}
  \hspace{1mm}
  \includegraphics[width=0.49\linewidth]{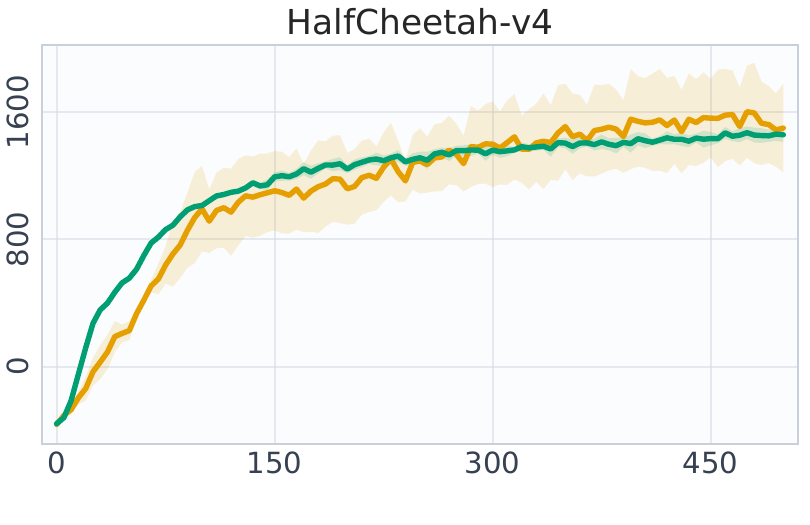}

  \vspace{0.2em}
  \includegraphics[width=0.49\linewidth]{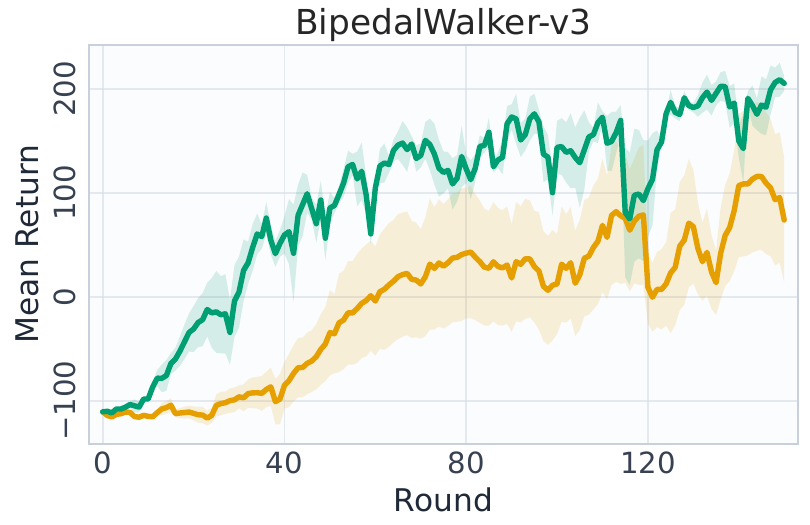}
  \hspace{1mm}
  \includegraphics[width=0.49\linewidth]{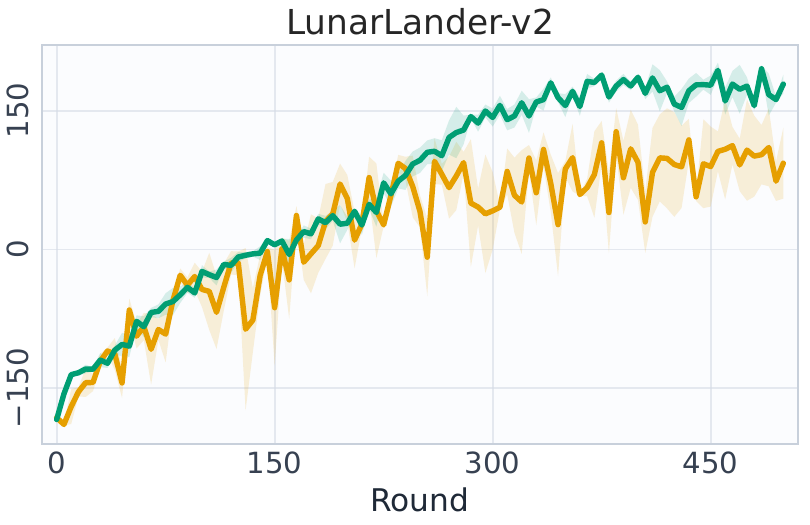}
  \caption{\textbf{Dynamic-horizon ablation.} Fixed horizon vs.\ the adaptive (dynamic) horizon. The dynamic horizon matches or outperforms the fixed variant on all four environments, with the largest gains on the harder-exploration tasks.}
  \label{fig:dynamic}
\end{figure}

\subsubsection{Effect of Horizon in RINSE}

We sweep the horizon over $H \in \{1, 2, 4, 5, 8, 10\}$ on \texttt{Ant-v4} and \texttt{HalfCheetah-v4}, holding all other hyperparameters fixed. Figure~\ref{fig:horizon} shows a consistent trend that emerges across both environments. As the horizon $H$ grows, the quality of the selected data improves gradually. This improvement, however, saturates at a moderate horizon of around $H=5$. Beyond this point, additional horizon brings no further benefit, and for the largest values it can even slightly degrade selection quality. This motivates some medium horizons, such as $H=5$, as a robust default.

\begin{figure}[t]
  \centering
  {\small \textcolor[HTML]{481D6F}{\raisebox{0.45ex}{\rule{1.6em}{1pt}}}~$H{=}1$\quad \textcolor[HTML]{3C4F8A}{\raisebox{0.45ex}{\rule{1.6em}{1pt}}}~$H{=}2$\quad \textcolor[HTML]{2A788E}{\raisebox{0.45ex}{\rule{1.6em}{1pt}}}~$H{=}4$\quad \textcolor[HTML]{1F9F88}{\raisebox{0.45ex}{\rule{1.6em}{1pt}}}~$H{=}5$\quad \textcolor[HTML]{4EC36B}{\raisebox{0.45ex}{\rule{1.6em}{1pt}}}~$H{=}8$\quad \textcolor[HTML]{B0DD2F}{\raisebox{0.45ex}{\rule{1.6em}{1pt}}}~$H{=}10$}\\[0.4em]
  \includegraphics[width=0.49\linewidth]{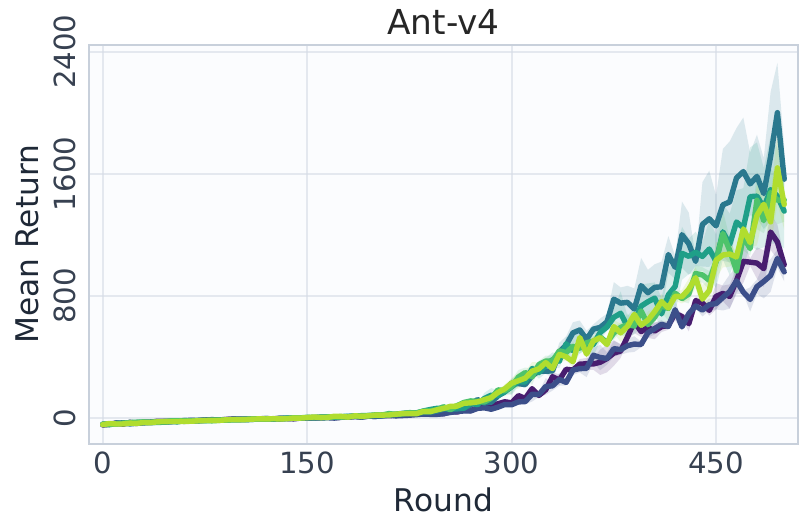}
  \hspace{1mm}
  \includegraphics[width=0.49\linewidth]{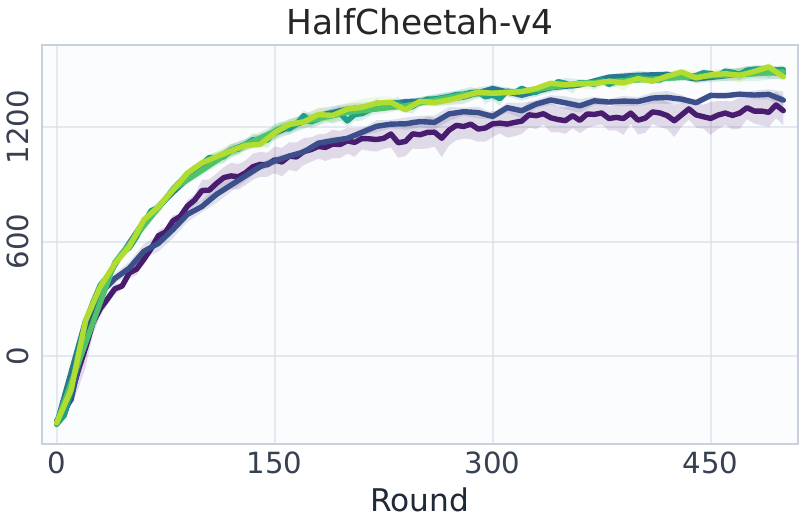}
  \caption{\textbf{Horizon ablation.} Effect of the different horizons $H \in \{1,2,4,5,8,10\}$. Longer horizons speed up early learning, but the benefit saturates around $H\!\approx\!4$--$5$.}
  \label{fig:horizon}
\end{figure}

\subsubsection{Effect of the Discard Ratio}

We ablate the discard ratio $\rho$ on \texttt{Ant-v4} and \texttt{HalfCheetah-v4}, sweeping $\rho \in \{0\%, 5\%, 10\%, 20\%, 50\%, 100\%\}$, where $\rho=0\%$ recovers standard PPO (no filtering) and $\rho=10\%$ is our default. Figure~\ref{fig:rho} shows that filtering improves performance across the tested ratios on \texttt{Ant-v4}, while \texttt{HalfCheetah-v4} is more sensitive to aggressive filtering. On \texttt{Ant-v4}, any non-zero discard ratio substantially outperforms the no-selection baseline, confirming that removing low-influence data is beneficial. On \texttt{HalfCheetah-v4}, $\rho=50\%$ and $\rho=100\%$ yield lower final mean returns than the no-selection baseline. These results support using a moderate discard ratio.

\begin{figure}[t]
  \centering
  {\small \textcolor[HTML]{444444}{\raisebox{0.45ex}{\rule{1.6em}{1pt}}}~$\rho{=}0\%$ (baseline)\quad \textcolor[HTML]{481D6F}{\raisebox{0.45ex}{\rule{1.6em}{1pt}}}~$\rho{=}5\%$\quad \textcolor[HTML]{375A8C}{\raisebox{0.45ex}{\rule{1.6em}{1pt}}}~$\rho{=}10\%$ (default)\quad \textcolor[HTML]{228B8D}{\raisebox{0.45ex}{\rule{1.6em}{1pt}}}~$\rho{=}20\%$\quad \textcolor[HTML]{3BBB75}{\raisebox{0.45ex}{\rule{1.6em}{1pt}}}~$\rho{=}50\%$\quad \textcolor[HTML]{B0DD2F}{\raisebox{0.45ex}{\rule{1.6em}{1pt}}}~$\rho{=}100\%$}\\[0.4em]
  \includegraphics[width=0.49\linewidth]{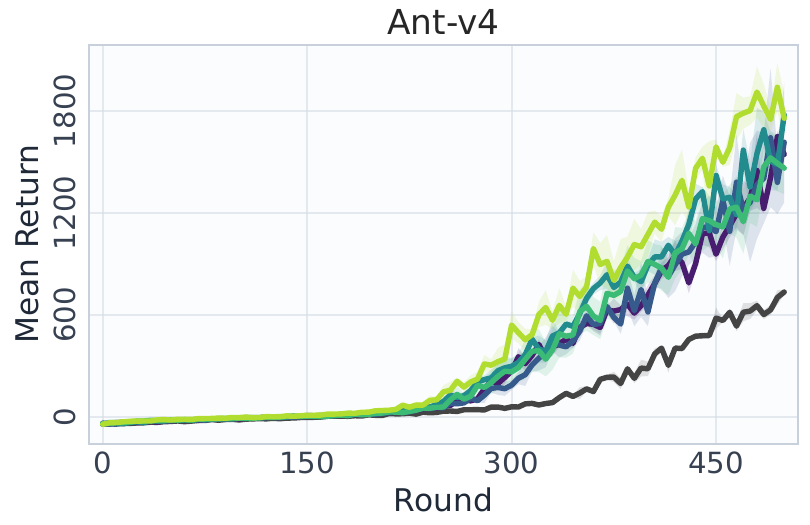}
  \hspace{1mm}
  \includegraphics[width=0.49\linewidth]{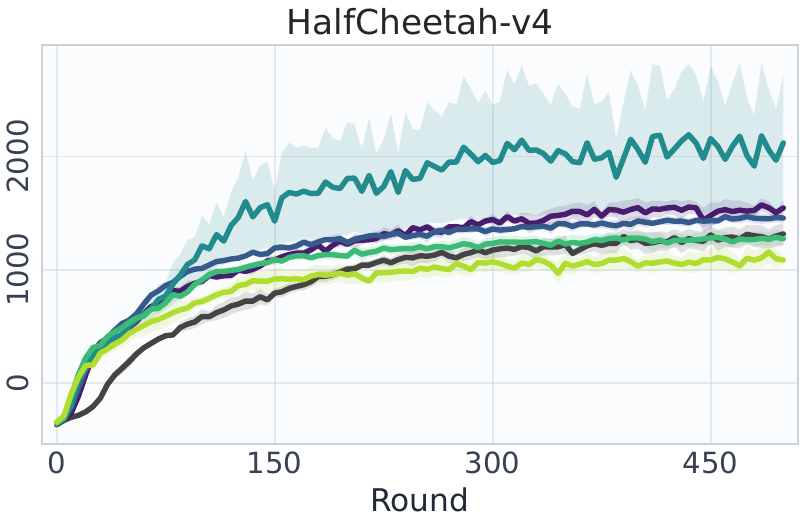}
  \caption{\textbf{Discard-ratio ablation.} Effect of the discard ratio $\rho$ on \texttt{Ant-v4} and \texttt{HalfCheetah-v4}. Moderate ratios ($\rho{\approx}10\%$--$20\%$) are best, while very aggressive filtering can degrade performance.}
  \label{fig:rho}
\end{figure}

\subsection{Further Details on Time and Sample Efficiency}
\label{app:efficiency-details}

We provide additional details on the computation of time efficiency (TE) and sample efficiency (SE), along with extended results. As described in the main text, both metrics are defined relative to the standard PPO baseline, which is normalized to $100\%$. Lower values therefore indicate better efficiency.

\paragraph{Evaluation protocol.}
For each method, we track the training trajectory and record the first time step at which a set of target performance levels is reached. In the main results, we use three thresholds corresponding to $90\%$, $95\%$, and $100\%$ of the peak performance achieved by the standard PPO baseline. TE and SE are then computed by averaging the relative runtime and data usage required to reach these thresholds.

In addition to this multi-level evaluation, we also report a stricter variant based on a single threshold at $100\%$ performance. This provides a complementary view that focuses on the final stage of training.

\paragraph{Results.}
Table~\ref{tab:efficiency-merged} summarizes both versions of the efficiency metrics. The upper block (``100\% only'') corresponds to the single-threshold evaluation, while the lower block (``90/95/100\%'') matches the results reported in the main text.

\begin{table*}[t]
\centering
\footnotesize
\caption{Relative time efficiency (TE) and sample efficiency (SE) across environments. All values are normalized to the PPO baseline ($100\%$); lower is better. Best results per column are in bold.}
\label{tab:efficiency-merged}
\resizebox{\linewidth}{!}{%
\begin{tabular}{lllcccccc}
\toprule
Setting & Metric & Method & Ant & HalfCheetah & LunarLander & BipedalWalker & Hopper & FrozenLake \\
\midrule
\multirow{4}{*}{100\% only}
& \multirow{2}{*}{TE$\downarrow$}
& IIF & 109.7{\tiny$\pm$2.0} & 77.5{\tiny$\pm$8.5} & \textbf{80.3{\tiny$\pm$6.0}} & 93.9{\tiny$\pm$3.3} & 114.7{\tiny$\pm$4.0} & 105.2{\tiny$\pm$7.9} \\
& & RINSE & \textbf{102.9{\tiny$\pm$2.4}} & \textbf{60.2{\tiny$\pm$6.7}} & 88.6{\tiny$\pm$3.9} & \textbf{71.1{\tiny$\pm$9.5}} & \textbf{107.5{\tiny$\pm$4.5}} & \textbf{104.2{\tiny$\pm$2.7}} \\
& \multirow{2}{*}{SE$\downarrow$}
& IIF & 85.8{\tiny$\pm$1.8} & 59.5{\tiny$\pm$6.5} & 55.4{\tiny$\pm$4.3} & 74.1{\tiny$\pm$2.8} & 87.9{\tiny$\pm$3.0} & 51.0{\tiny$\pm$3.9} \\
& & RINSE & \textbf{74.7{\tiny$\pm$1.8}} & \textbf{42.8{\tiny$\pm$4.7}} & \textbf{49.1{\tiny$\pm$2.1}} & \textbf{49.7{\tiny$\pm$7.0}} & \textbf{77.2{\tiny$\pm$3.4}} & \textbf{48.0{\tiny$\pm$1.6}} \\
\midrule
\multirow{4}{*}{90/95/100\%}
& \multirow{2}{*}{TE$\downarrow$}
& IIF & 110.2{\tiny$\pm$3.7} & 75.0{\tiny$\pm$11.0} & \textbf{76.1{\tiny$\pm$8.5}} & 94.9{\tiny$\pm$6.0} & 94.5{\tiny$\pm$18.8} & 133.4{\tiny$\pm$26.0} \\
& & RINSE & \textbf{103.0{\tiny$\pm$4.9}} & \textbf{63.9{\tiny$\pm$10.4}} & 89.9{\tiny$\pm$6.8} & \textbf{67.9{\tiny$\pm$13.6}} & \textbf{88.9{\tiny$\pm$16.7}} & \textbf{116.3{\tiny$\pm$12.5}} \\
& \multirow{2}{*}{SE$\downarrow$}
& IIF & 86.3{\tiny$\pm$3.3} & 57.7{\tiny$\pm$8.5} & 52.6{\tiny$\pm$6.0} & 74.8{\tiny$\pm$5.0} & 72.2{\tiny$\pm$14.6} & 64.2{\tiny$\pm$12.2} \\
& & RINSE & \textbf{74.9{\tiny$\pm$3.7}} & \textbf{45.4{\tiny$\pm$7.3}} & \textbf{49.8{\tiny$\pm$3.7}} & \textbf{47.1{\tiny$\pm$9.7}} & \textbf{63.9{\tiny$\pm$11.9}} & \textbf{53.5{\tiny$\pm$5.0}} \\
\bottomrule
\end{tabular} }
\end{table*}

Across both evaluation protocols, RINSE improves mean sample efficiency over IIF on all six environments and mean time efficiency on five of the six environments. The gains are particularly pronounced in continuous control tasks such as Ant and HalfCheetah, where RINSE reduces sample usage by a substantial margin.

Comparing the two settings, the multi-threshold evaluation (90/95/100\%) yields slightly more stable estimates, as it averages performance across different stages of training. In contrast, the $100\%$-only metric is more sensitive to late-stage behavior but leads to similar overall conclusions.

Overall, these results suggest that the efficiency improvements of RINSE are not limited to a specific performance regime, but persist throughout the training process.

\subsubsection{Wall-Clock Runtime Breakdown}
\label{app:experiment-rt}

Table~\ref{tab:runtime-control} reports the wall-clock runtime per round for each control environment, separating rollout collection, the initial training pass, attribution, and filtered replay. The baseline cost is the sum of collection and training, while the additional cost is the sum of attribution and replay. The final column reports this additional cost relative to the baseline.

\begin{table}[htbp]
    \centering
    \caption{ Wall-clock runtime in seconds per round on the traditional RL environments. }
    \label{tab:runtime-control}
    \resizebox{\linewidth}{!}{%
    \begin{tabular}{lrrrrrrr}
        \toprule
        Environment & Collect & Train & Attribution & Replay
        & Base & Overhead & Overhead / Base \\
        \midrule
        Ant-v4
        & 1.46 & 17.66 & 15.90 & 5.51 & 19.12 & 21.41 & 112\% \\
        HalfCheetah-v4
        & 0.86 & 17.13 & 13.85 & 5.36 & 17.99 & 19.21 & 107\% \\
        Hopper-v4
        & 1.06 & 26.45 & 14.28 & 5.66 & 27.51 & 19.94 & 72\% \\
        BipedalWalker-v3
        & 7.24 & 26.95 & 15.56 & 5.66 & 34.19 & 21.22 & 62\% \\
        LunarLander-v2
        & 0.75 & 0.97 & 0.62 & 1.44 & 1.72 & 2.06 & 120\% \\
        FrozenLake-v1
        & 1.43 & 3.20 & 3.07 & 2.87 & 4.63 & 5.94 & 128\% \\
        \bottomrule
    \end{tabular}%
    }
\end{table}

Attribution and replay add between $62\%$ and $128\%$ of the baseline per-round cost in these measurements. This overhead differs from TE, which measures total time required to reach specified performance thresholds. Faster progress in return can offset the additional per-round computation, but the resulting wall-clock benefit depends on the environment.

\subsection{Additional Results for RLHF}
\label{app:results-rlhf}

We provide additional details on RLHF training dynamics and efficiency.

\begin{figure}[H]
    \centering
    \includegraphics[width=0.9\linewidth]{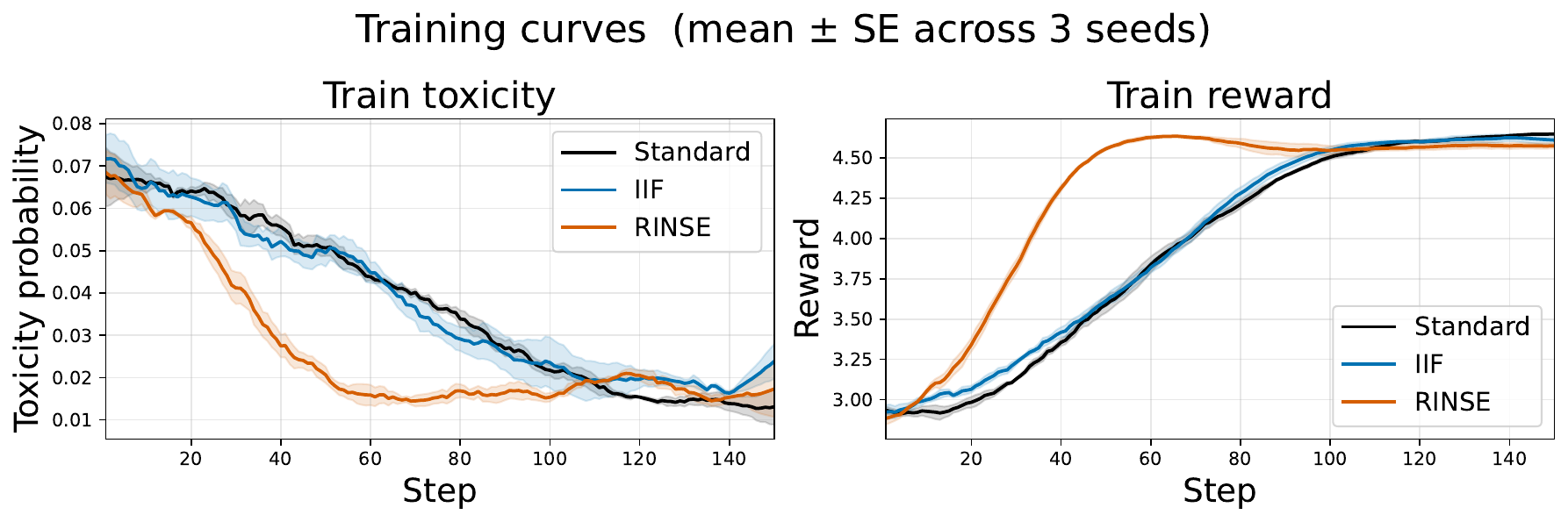}
    \caption{Training curves of Standard, IIF, and RINSE on RLHF toxicity reduction.}
    \label{fig:rlhf-train-curves}
\end{figure}

\paragraph{Training dynamics.}
Figure~\ref{fig:rlhf-train-curves} shows the training curves of Standard, IIF, and RINSE on training toxicity and reward. Standard and IIF exhibit similar learning dynamics, with IIF slightly faster than Standard, while RINSE achieves more significant toxicity reduction together with higher reward under the same number of training steps.

\begin{figure}[H]
    \centering
    \begin{minipage}{0.48\linewidth}
        \centering
        \includegraphics[width=\linewidth]{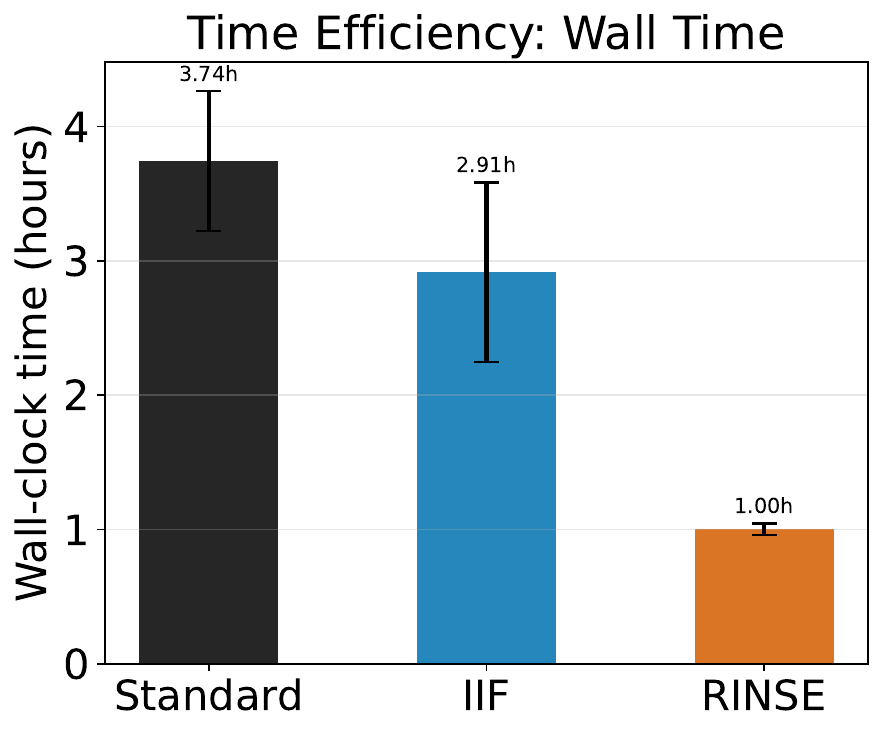}
        \caption{Average wall time of Standard, IIF, and RINSE on RLHF.}
        \label{fig:rlhf-runtime-efficiency}
    \end{minipage}
    \hfill
    \begin{minipage}{0.48\linewidth}
        \centering
        \includegraphics[width=\linewidth]{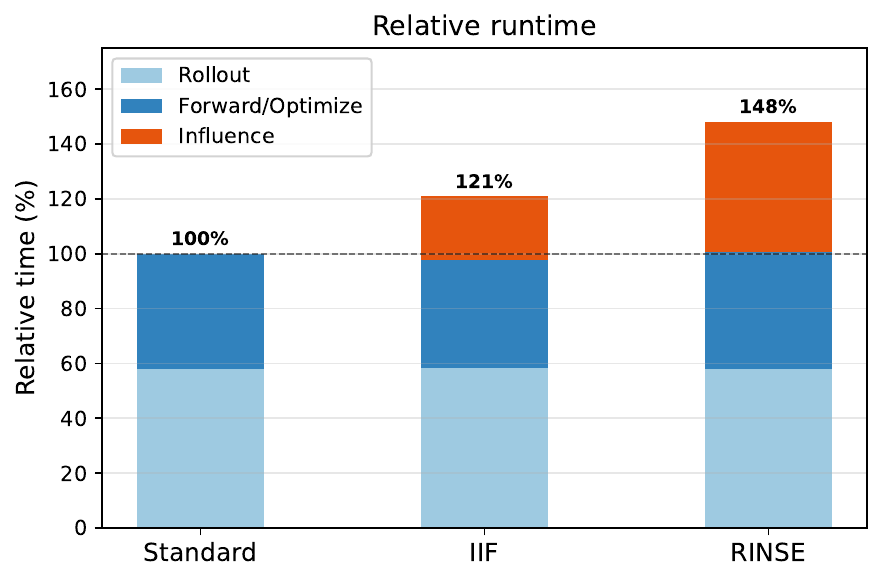}
        \caption{Per-step wall-clock time breakdown for Standard, IIF, and RINSE on RLHF.}
        \label{fig:rlhf-runtime-breakdown}
    \end{minipage}
\end{figure}

\paragraph{Time efficiency and runtime breakdown.}
Figure~\ref{fig:rlhf-runtime-efficiency} reports the wall-clock time required for each method to first reach the best test toxicity achieved by Standard on each seed. RINSE reaches the target substantially faster than Standard, demonstrating that non-local data filtering accelerates convergence.

Figure~\ref{fig:rlhf-runtime-breakdown} further breaks down the per-step wall-clock time by training phase. With the Generalized Gauss--Newton matrix~\cite{martens2020new} approximation, RINSE is only \textasciitilde$1.5\times$ slower than standard training per step. Although RINSE introduces this additional computation per training step, the overhead is compensated by the reduction in required training steps, leading to improved overall wall-clock efficiency.

\subsubsection{Generation Quality at Matched Toxicity}
\label{app:rlhf-quality}

We evaluate whether the reduction in toxicity is accompanied by a degradation in generation quality on held-out Wiki-Toxic prompts. We compare RINSE with the unperturbed baseline at a matched toxicity level, so that differences in quality are not explained solely by different degrees of toxicity reduction.

We assess fluency using the perplexity of the generated continuation under a separate language model, GPT-2 Large~\citep{radford2019language}, conditioned on the prompt. This evaluator is independent of both the policy and the training reward model. We assess diversity using corpus-level Distinct-2 and Distinct-3, and repetition using the within-sample repeated-4-gram rate, Seq-rep-4.

\begin{table}[htbp]
    \centering
    \caption{ Generation quality on held-out Wiki-Toxic prompts at matched toxicity. }
    \label{tab:rlhf-quality}
    \begin{tabular}{lcc}
        \toprule
        Metric & RINSE & Unperturbed baseline \\
        \midrule
        Median perplexity $\downarrow$
        & $\mathbf{43.4 \pm 2.6}$ & $50.5 \pm 5.3$ \\
        Distinct-2 $\uparrow$
        & $\mathbf{0.79 \pm 0.01}$ & $0.72 \pm 0.03$ \\
        Distinct-3 $\uparrow$
        & $\mathbf{0.96 \pm 0.00}$ & $0.92 \pm 0.02$ \\
        Seq-rep-4 $\downarrow$
        & $\mathbf{0.0037 \pm 0.0013}$ & $0.0046 \pm 0.0010$ \\
        \bottomrule
    \end{tabular}
\end{table}

At matched toxicity, RINSE achieves lower median perplexity, higher Distinct-2 and Distinct-3, and a lower repeated-4-gram rate than the unperturbed baseline. These results provide evidence that the toxicity improvement does not come at the cost of the evaluated fluency and diversity measures.

\newpage
\section*{NeurIPS Paper Checklist}

\begin{enumerate}

\item {\bf Claims}
\item[] Question: Do the main claims made in the abstract and introduction accurately reflect the paper's contributions and scope?
\item[] Answer: \answerYes{}
\item[] Justification: We tried our best to ensure that the abstract and introduction accurately reflect the paper's contributions and scope.

\item {\bf Limitations}
\item[] Question: Does the paper discuss the limitations of the work performed by the authors?
\item[] Answer: \answerYes{}
\item[] Justification: We discuss the limitations in Section~\ref{sec:conclusions}.

\item {\bf Theory assumptions and proofs}
\item[] Question: For each theoretical result, does the paper provide the full set of assumptions and a complete (and correct) proof?
\item[] Answer: \answerYes{}{}
\item[] Justification: We provide the theorems along with the assumptions and proof in Appendix~\ref{app:theorem-derivations}

\item {\bf Experimental result reproducibility}
\item[] Question: Does the paper fully disclose all the information needed to reproduce the main experimental results of the paper to the extent that it affects the main claims and/or conclusions of the paper (regardless of whether the code and data are provided or not)?
\item[] Answer: \answerYes{}
\item[] Justification: We provide detailed information on the experimental setups in Appendix~\ref{app:experimental-setups}, and release our code at \url{https://github.com/sx-liu/RL-IF}.

\item {\bf Open access to data and code}
\item[] Question: Does the paper provide open access to the data and code, with sufficient instructions to faithfully reproduce the main experimental results, as described in supplemental material?
\item[] Answer: \answerYes{}
\item[] Justification: Our code is available at \url{https://github.com/sx-liu/RL-IF}.

\item {\bf Experimental setting/details}
\item[] Question: Does the paper specify all the training and test details (e.g., data splits, hyperparameters, how they were chosen, type of optimizer) necessary to understand the results?
\item[] Answer: \answerYes{}
\item[] Justification: We discuss the results in Section~\ref{sec:data+selection} and provide the detailed setups in Appendix~\ref{app:experimental-setups}

\item {\bf Experiment statistical significance}
\item[] Question: Does the paper report error bars suitably and correctly defined or other appropriate information about the statistical significance of the experiments?
\item[] Answer: \answerYes{}
\item[] Justification: We use at least 3 random seeds in all of our experiments. We report the standard errors in tables and include error bars in figures.

\item {\bf Experiments compute resources}
\item[] Question: For each experiment, does the paper provide sufficient information on the computer resources (type of compute workers, memory, time of execution) needed to reproduce the experiments?
\item[] Answer: \answerYes{}{}
\item[] Justification: We discuss the compute in Appendix~\ref{app:experimental-setups}, and also provide detailed runtime analysis in Section~\ref{sec:data+selection}, Appendix~\ref{app:efficiency-details} and Appendix~\ref{app:results-rlhf}

\item {\bf Code of ethics}
\item[] Question: Does the research conducted in the paper conform, in every respect, with the NeurIPS Code of Ethics \url{https://neurips.cc/public/EthicsGuidelines}?
\item[] Answer: \answerYes{}
\item[] Justification: To the best of our knowledge, the research presented in this paper conforms with the NeurIPS Code of Ethics.

\item {\bf Broader impacts}
\item[] Question: Does the paper discuss both potential positive societal impacts and negative societal impacts of the work performed?
\item[] Answer: \answerNA{}
\item[] Justification: This paper conducts fundamental research in theory, and there is no societal impact of the work performed.

\item {\bf Safeguards}
\item[] Question: Does the paper describe safeguards that have been put in place for responsible release of data or models that have a high risk for misuse (e.g., pre-trained language models, image generators, or scraped datasets)?
\item[] Answer: \answerNA{}
\item[] Justification: This paper does not release high-risk models, datasets, or generative systems.

\item {\bf Licenses for existing assets}
\item[] Question: Are the creators or original owners of assets (e.g., code, data, models), used in the paper, properly credited and are the license and terms of use explicitly mentioned and properly respected?
\item[] Answer: \answerYes{}
\item[] Justification: We have cited the datasets, models, code frameworks and included their licenses in Appendix~\ref{app:experimental-setups}.

\item {\bf New assets}
\item[] Question: Are new assets introduced in the paper well documented and is the documentation provided alongside the assets?
\item[] Answer: \answerNA{}
\item[] Justification: This paper does not release new assets.

\item {\bf Crowdsourcing and research with human subjects}
\item[] Question: For crowdsourcing experiments and research with human subjects, does the paper include the full text of instructions given to participants and screenshots, if applicable, as well as details about compensation (if any)?
\item[] Answer: \answerNA{}
\item[] Justification: This paper does not include crowdsourcing experiments or research with human subjects.

\item {\bf Institutional review board (IRB) approvals or equivalent for research with human subjects}
\item[] Question: Does the paper describe potential risks incurred by study participants, whether such risks were disclosed to the subjects, and whether Institutional Review Board (IRB) approvals (or an equivalent approval/review based on the requirements of your country or institution) were obtained?
\item[] Answer: \answerNA{}
\item[] Justification: This paper does not include research with human subjects.

\item {\bf Declaration of LLM usage}
\item[] Question: Does the paper describe the usage of LLMs if it is an important, original, or non-standard component of the core methods in this research? Note that if the LLM is used only for writing, editing, or formatting purposes and does \emph{not} impact the core methodology, scientific rigor, or originality of the research, declaration is not required.
\item[] Answer: \answerNA{}
\item[] Justification: This paper does not involve LLMs as any important, original, or non-standard components.

\end{enumerate}

\end{document}